\documentclass[preprint,11pt]{elsarticle}

\usepackage{amssymb}
\usepackage{amsthm}

\usepackage{xcolor}
\usepackage{amssymb}
\usepackage{amsmath}
\usepackage{comment}
\usepackage{algorithm}
\usepackage[noend]{algpseudocode}

\usepackage{ulem}
\usepackage{tikz}
\usetikzlibrary{patterns}
\renewcommand{\and}{\mbox{ and }}
\newcommand{\inc}{\sim}

\usepackage{thm-restate}
\usepackage{mathrsfs}
\usepackage{comment}
\usepackage{environ}
\usepackage{balance} %
\usepackage{booktabs}
\usepackage{multirow}
\usepackage[acronym]{glossaries}
\usepackage{url}
\usepackage{enumitem}
\usepackage[utf8]{inputenc}  
\usepackage[T1]{fontenc}  

\newtheorem{example}{Example}
\newtheorem{theorem}{Theorem}
\newtheorem{corollary}{Corollary}
\newtheorem{remark}{Remark}

\newcommand{\AlternativeSet}{\mathcal{A}}

\newcommand{\TierFunction}{r}
\newcommand{\lex}{\mathrm{lex}}
\newcommand{\opt}{\mathrm{opt}}
\newcounter{constraint}
\newcommand{\constraintlabel}[1]{\refstepcounter{constraint}\label{#1}(\theconstraint)}

\newtheorem{definition}{Definition}

\NewEnviron{cproblem}[1]{%
\begin{center}\fbox{\parbox{5in}{%
    {\centering\scshape #1\par}%
    \parskip=1ex
    \BODY
}}\end{center}}

\excludecomment{versiona}

\newlength{\bibitemsep}
\newlength{\bibparskip}
\let\oldthebibliography\thebibliography
\renewcommand\thebibliography[1]{%
  \oldthebibliography{#1}%
  \setlength{\parskip}{\bibitemsep}%
  \setlength{\itemsep}{\bibparskip}%
}

\journal{arXiv}%

\usepackage{nomencl}
\makenomenclature

\expandafter\def\expandafter\normalsize\expandafter{%
    \normalsize%
    \setlength\abovedisplayskip{2pt}%
    \setlength\belowdisplayskip{2pt}%
    \setlength\abovedisplayshortskip{-8pt}%
    \setlength\belowdisplayshortskip{2pt}%
}

\begin{document}

\begin{frontmatter}

\title{Robust Ordinal Regression for Subsets Comparisons with Interactions\tnoteref{t1}}

\tnotetext[t1]{This paper is a revised and extended version of a workshop paper at MPREF 2022, and an extended abstract at AAMAS 2023:\\ H. Gilbert, M. Ouaguenouni, M. \"Ozt\"urk, O. Spanjaard, \emph{Cautious Learning of Multiattribute Preferences}, 13th Workshop MPREF, Jul 2022, Vienna, Austria.\\ H. Gilbert, M. Ouaguenouni, M. \"Ozt\"urk, O. Spanjaard,
\emph{Robust Ordinal Regression for Collaborative Preference Learning with Opinion Synergies}, AAMAS 2023, pp. 2439-2441.}

\affiliation[inst1]{organization={Sorbonne Université, CNRS, LIP6},%
            city={Paris},
            postcode={F-75005}, 
            country={France}}
            
\affiliation[inst2]{organization={Université Paris Dauphine, PSL Research University, CNRS, LAMSADE},%
            city={Paris},
            postcode={F-75016}, 
            country={France}}

\author[inst2]{Hugo Gilbert}
\author[inst1]{Mohamed Ouaguenouni\corref{cor1}}
\author[inst2]{Meltem \"Ozt\"urk\corref{cor2}}
\author[inst1]{Olivier Spanjaard}
\cortext[cor1]{Corresponding author}
\cortext[cor2]{A significant part of the work presented here has been carried out while Meltem \"Ozt\"urk was on delegation at LIP6.}

\begin{abstract}
This paper is dedicated to a robust ordinal method for learning the preferences of a decision maker between subsets. 
The decision model, derived from Fishburn and LaValle \cite{Fishburn1996BinaryIA} and whose parameters we learn, is general enough to be compatible with any strict weak order on subsets, thanks to the consideration of possible interactions between elements.
Moreover, we accept not to predict some preferences if the available preference data are not compatible with a reliable prediction. 
A predicted preference is considered reliable if all the simplest models (Occam’s razor) explaining the preference data agree on it. 
Following the robust ordinal regression methodology, our predictions are based on %
an uncertainty set encompassing the possible values of the model parameters.
We define a robust ordinal dominance relation between subsets and we design a procedure to determine whether this dominance relation holds.
Numerical tests are provided on synthetic and real-world data to evaluate the richness and reliability of the preference predictions made.%
\end{abstract}


\end{frontmatter}

\section{Introduction}\label{sec:intro}

Preference elicitation (or preference learning) is an important step in setting up a recommender system for decision making. 
In this preference elicitation setting, our focus is on determining the parameters of a decision model that accurately captures the pairwise preferences of a Decision Maker (DM) over subsets, by comparing subsets of elements.
The preferences are depicted using a highly adaptable model whose versatility stems from its ability to incorporate positive or negative synergies between elements \cite{grabisch2008review}.
Moreover, we provide an ordinally robust approach, in the sense that the preferences we infer do not rely on arbitrarily specified parameter values, but on the set of all parameter values that are compatible with the observed preferences. 
Importantly, another distinctive feature of our approach is its ability to learn the parameter set itself (not only the \emph{values} of parameters).

The preference model we consider can be used in different contexts, depending on the nature of the subsets we are comparing. 
The subsets are represented by binary vectors, showing the presence or absence of an element in the subset. The elements of a subset can be for example:
\begin{itemize}[noitemsep,topsep=2pt]
\item individuals (in the comparison of coalitions, teams, etc.), 
\item binary attributes (in the comparison of multiattribute alternatives), 
\item objects (in the comparison of subsets in a subset choice problem), etc.
\end{itemize} 

For illustration, a toy example of such an elicitation context could be a coffee shop trying to determine its customers' favorite frozen yogurt flavor combination by offering them to test a small number of flavor combinations rather than having them taste each combination. %

\paragraph{Objective of the paper} Our objective is to design a preference elicitation procedure that complies with the two following principles.

First, the sophistication of the learned preference model should be able to fit any level of complexity of the stated preferences. For this purpose,  we use a utility function $f$ general enough to represent any order $\succ$ of preference, i.e., for any strict weak ordering $\succ$ on a set $\AlternativeSet$ of alternatives (i.e., subsets) there exists $f$ such that, for any pair $\{A,B\}\!\subseteq\!\AlternativeSet$, $f(A)\!>\!f(B)$ iff $A\!\succ\!B$. Note that we also aim to make the model as simple as possible, in the sense that the parameter set remains as concise as possible (\emph{sparse} model).%
    
    Second, the predicted pairwise preferences should not depend on the partly arbitrary choice of precise numerical values for the parameters of the model but solely on the stated preferences. Hence, we design an \emph{ordinally robust} elicitation procedure that maintains an isomorphism between the collected preferential data and the learned model (in the same spirit as ordinal measurement for problem solving \cite{bartee1971problem}) by using a polyhedron of possible values for the parameters, reflecting the uncertainty about them.
As a consequence, when predicting an unknown pairwise preference between two alternatives $A$ and $B$, apart from the predictions ``$A$ is preferred to $B$'' and ``$B$ is preferred to $A$'', it is possible that the model does not make a prediction due to a lack of sufficiently rich preferential data (the absence of prediction is preferred to a wrong prediction, although a compromise  must obviously be made between the reliability of the prediction and the predictive power of the learned model).

\paragraph{Elicitation setting} 
The input of our elicitation procedure is a learning set consisting of pairwise comparisons of various alternatives. More precisely, we consider an offline elicitation setting (passive learning) where we assume that a dataset of comparison examples is available, from which the parameters of the preference model are (partially) specified. This is a separate framework from the online elicitation setting (active learning) where we would incrementally select pairwise preference queries to enrich the learning set. The output of the elicitation procedure consists of pairwise comparisons that were not present in the learning set, which we call (preference) \emph{predictions} hereafter. Note that, in some cases, the model may choose not to provide a prediction. The elicitation procedure thus results in a strict partial order on the alternatives.

\paragraph{Organization of the paper} After an overview of the related work (Section~\ref{sec:related}), we present the $\theta$-additive utility model (Section~\ref{sec:sparse}), as well as the robust ordinal dominance relation inferred from it, based on the knowledge of a collection of preference examples. We then show how to determine whether a subset dominates another subset given the known pairwise preferences of the DM (Section~\ref{sec:PrefPrediction}), which enables to make preference predictions. The paper ends with numerical tests on synthetic and real-world preference data, and comparison with other preference learning methods (Section~\ref{sec:tests}).

\section{Related work}\label{sec:related}

Preference elicitation (see e.g. Dias et al.~\cite{Diasetal2018}) and preference learning (see e.g. Fürnkranz and Hüllermeier~\cite{furnkranz2003pairwise}, Corrente et al.~\cite{ corrente2013robust}) have been studied for a long time in operations research and artificial intelligence. This is a prerequisite in many applications across a wide range of fields, such as recommender systems, banking, financial management, chemistry, energy resources, health, investments, and industrial location \cite{Andreopoulou2017}.  
Several issues can be tackled in preference elicitation, among which:
\begin{enumerate}[noitemsep,topsep=2pt]
\item to handle a set of alternatives of combinatorial nature: an incremental preference elicitation is then often adopted, where comparison examples are interactively generated with the DM, in order to determine a necessary ``optimal'' alternative \cite[e.g.,][]{BenabbouLLP21, BOUTILIER2006686, wangincr};
\item to cope with preferences that cannot be represented by an additive utility function: for instance, the elicitation of generalized utility functions has been considered in the literature \cite{braziunas2007}, but also the elicitation of several other involved decision models \cite{benabbounonadd,Schmeidler1986IntegralRW};
\item to deal with ``incorrect'' preference examples: Bayesian approaches have been considered in this matter \cite{bourdache2019, guo2010multiattribute}, but also possibilistic approaches \cite{adam2021possibilistic}.
\end{enumerate}

We focus here on the second challenge, by studying the elicitation of 
a set function taking into account positive and negative interactions between elements. Furthermore, preference elicitation problems differ in their purpose: some aim to produce a recommendation, others a set of recommendations, and still others pairwise comparisons. We will, in our case, produce a set of pairwise comparisons.

The Choquet integral is the most studied decision model for taking into account positive and negative interactions between criteria in multicriteria decision making \cite{grabisch1996application}. It turns out that a Choquet integral defined on binary vectors representing subsets can be viewed as a set function. 
Note that a Choquet integral is parameterized by a capacity $v$ on the criteria set $N$, i.e., a set function on $N$ that is \emph{monotone} ($A\!\subseteq\!B \Rightarrow v(A) \!\leq\!v(B)$) and \emph{normalized} ($v(N)\!=\!1$). As will become clear in the remainder of the paper, we do not impose such constraints in the model we consider.
There are some recent works dealing with the elicitation of the parameters of a Choquet-related aggregation function: Bresson et al.~\cite{Bresson2020Learning2H} use a perceptron approach to learn the parameters of a 2-additive hierarchical Choquet integral, while Herin et al.~\cite{herin2023learning} propose an algorithm to learn sparse M\"obius representations from preference examples, without a prior $k$-additivity assumption.
For a broad literature review about learning the parameters of a Choquet integral, the reader may refer to the article by Grabisch, Kojadinovic, and Meyer~\cite{grabisch2008review}.
Let us also mention the work by Marichal and Roubens~\cite{marichal2000determination}, which use a polyhedron to characterize the set of parameters that are compatible with a training set of examples. The idea of defining a polyhedron of uncertainty on the parameters of a utility function goes back at least to the work of Charnetski and Soland~\cite{charnetski1978multiple}. %
Their model state that $A\!\succ\!B$ if the proportion of parameters that give a better value for $A$ than for $B$ among those that are compatible with the stated preferences is greater than the proportion of parameters that give a better value for $B$ than for $A$.
This principle was also adapted to the case of a Choquet integral by Angilella, Corrente and Greco~\cite{Angilella2015}. %
In the sequel, 
we will use a similar polyhedron. %

More precisely, we elicit a partial specification of a set function, namely the components of the parameter set and the set of parameter values, which yields an ordinal dominance relation between subsets. %
As already mentioned, we do not assume interactions with the DM but only the knowledge of a ``static'' training set of examples of pairwise preferences in order to predict pairwise comparisons between alternatives.

Predicting a comparison between alternatives can be framed as a binary classification problem by considering, as a training set, a set of triples $(A,B,c)$, where $A$ and $B$ are two alternatives and $c\!=\!1$ if $A\!\succ\!B$, and $c\!=\!0$ otherwise. In this setting, many approaches have been proposed, going from perceptrons \cite{abs-1711-07875} to Gaussian processes \cite{chu05preference} or Support Vector Machines (SVM) \cite{domshlak2005unstructuring}.

An important feature of our elicitation procedure is that it may lead to not making predictions for some pairwise comparisons if the available preferential information is not conclusive enough. Other classification models also have  such a possibility to not predict a class for some examples, either because of an ambiguity in the class to predict (ambiguity rejection) or because the example is too far from the examples that are in the learning set (novelty rejection). This type of approaches are generally used in safety-sensitive domains, e.g. to predict a disease in medical applications \cite{Kompa2021}. For a complete review of learning with reject option, we refer the reader to the survey made by Hendrickx et al. \cite{rejectsurvey}.

The two closest works to ours are those by Domshlak and Joachims~\cite{domshlak2005unstructuring} and by Bigot et al.~\cite{bigot2012using}. Similarly to our approach, Domshlak and Joachims consider a function that could represent any weak order on the alternatives. More precisely, they consider a multiattribute utility function that is a sum of $4^n$ subutilities over subsets of attribute values, where $n$ is the number of attributes. The subutility values are then learned using an efficient SVM approach based on the \emph{kernel trick} \cite[see e.g.,][]{scholkopf2018learning}.
Bigot et al. study the use of generalised additively independent decompositions of utility functions \cite{fishburn1970utility,gonzales:hal-01492604}. They give a PAC-learner that is polynomial time if a constant bound is known on the degree of the function, where the \emph{degree} is the size of the greatest subset of attributes in the decomposition. Yet, both works do not fit the robust ordinal learning framework we consider in this article.

\section{From the $\theta$-additive model to robust ordinal dominance}\label{sec:sparse}

Given a set $\mathcal{F}\!=\!\{a_1, a_2, \ldots, a_n\}$ of elements, we aim to reason on the preferences of the DM on a set $\AlternativeSet$ of subsets $A\!\subseteq\!\mathcal{F}$, representing alternatives. The characteristic vector $\overrightarrow{A}$ of a subset $A$ is the $n$-dimensional binary vector whose $i^{\mbox{\tiny th}}$ component is 1 if $a_i\!\in\!A$, and 0 otherwise. 
For instance, the characteristic vector of $A\!=\!\{a_1,a_2,a_4\}$ is $\overrightarrow{A}\!=\!(1,1,0,1)$ if $\mathcal{F}\!=\!\{a_1, a_2, a_3, a_4\}$.
In the following, we may use one or the other notation for describing a subset. %
Here are some examples of alternatives represented by subsets:

\begin{itemize}[noitemsep,topsep=2pt]

\item If $\mathcal{F}$ is a set of reference users expressing opinions on cultural products (e.g., movies), a cultural product may be represented by the subset $A$ of reference users in $\mathcal{F}$ that have a positive opinion on it, i.e., $a_i\!\in\!A$ if reference user $a_i$ has a positive opinion on it, otherwise $a_i\!\not\in\!A$. 
\item If $\mathcal{F}$ is the set of players in a squad, a team lineup may be represented by the subset $A$ of players that compound it. %
\item If $\mathcal{F}$ is a set of binary features of technological products (e.g., smartphones), a technological product may be represented by a subset $A$ of features, i.e., $a_i\!\in\!A$ if the product has feature $a_i$, otherwise $a_i\!\not\in\!A$.
\end{itemize}
 
We assume for simplicity that there are no two distinct alternatives corresponding to the same subset $A\!\subseteq\!\mathcal{F}$, which implies in particular that $2^{|\mathcal{F}|}\!\ge\!|\mathcal{A}|$. %
We infer strict pairwise preferences from strict preferences given by a DM on some subset of alternatives in $\AlternativeSet$, and we use this training set of pairwise preferences on alternatives (each viewed as a subset) to elicit the parameters of a utility function $f$ defined on $\mathcal{A}$. The role of the utility function $f$ is to represent the (unknown) strict weak order on $\mathcal{A}$ corresponding to the DM's preferences, with $A\!\succ\!B$ iff $f(A)\!>\!f(B)$ and $A\!\sim\!B$ iff $f(A)\!=\!f(B)$.%

We do not perform a full elicitation of the parameters of $f$, but we consider an uncertainty set of parameters values consistent with the known preferences of the DM, as in robust ordinal regression. If $f(A)\!>\!f(B)$ for all parameters values in this uncertainty set, then $A$ is predicted to be strictly preferred to $B$. Actually, we do not only learn the parameters values, but also the components of the parameter set themselves, as we explain below.

\subsection{The $\theta$-additive model}

Before coming to the proposed $\theta$-additive model, we first recall the standard additive utility model, and its extension, the $k$-additive utility model.

\paragraph{The additive and $k$-additive utility models} As the DM's preferences over $\AlternativeSet$ are modeled as a strict weak order, there exists a real-valued function $f$ such that $\forall A,B\!\in\!\AlternativeSet, f(A)\!>\!f(B)\!\Leftrightarrow\!A\!\succ\!B$. 
Many models assume that $f$ can be represented in a compact way using some sort of additivity property. 
The simplest and most used one is the additive model \cite{fishburn1970utility}.
This model makes the strong assumption that we can find a parameter value $v(a)\!\in\!\mathbb{R}$ for each element $a\!\in\!\mathcal{F}$ such that for all $A\!\in\!\mathcal{A}$, the utility of $A$ is $f(A)\!=\!\sum_{a \in A } v(a)$. %
This assumption is strong because it implies that there is no interaction between the elements.  
A weaker assumption is that of \emph{$k$-additivity} where we suppose the existence of a parameter $v(S)\!\in\!\mathbb{R}$ for each $S\!\in\![\mathcal{F}]^k$, where $[\mathcal{F}]^k\!=\!\{S\!\subseteq\mathcal{F}\!:\!1\!\le\!|S|\!\le\!k\}$. 
Hence, in the $k$-additive model, for all $A\!\in\!\AlternativeSet$, $f(A)\!=\!\sum_{S \in [\mathcal{F}]^k} I_A(S)v_S$, where $I_A(S)\!=\!1$ if $S \!\subseteq\!A$ and 0 otherwise, and $v_S$ is an abbreviation for $v(S)$. Obviously, the $1$-additive model amounts to the additive model. Taking $k$ strictly greater than 1 makes it possible to account for (positive or negative) synergies between subsets of $k$ or less elements.
For example, the $2$-additive model makes it possible to account for binary synergies. 
The utility of the alternative $A\!=\!(1,1,0,1)$ with the 2-additive model is $f(A)\!=\!v(\{a_1\})+v(\{a_2\})+v(\{a_4\})+v(\{a_1,a_2\})+v(\{a_1,a_4\})+v(\{a_2,a_4\})$. %
If there is a positive synergy between $a_1$ and $a_2$ then $f(\{a_1,a_2\})> v(\{a_1\})+v(\{a_2\})$ holds because $f(\{a_1,a_2\})\!=\!v(\{a_1\})+v(\{a_2\})+v(\{a_1,a_2\})$. Note incidentally that $f(\{a_1,a_2\})\!\neq\!v(\{a_1,a_2\})$. 
The $n$-additive model is general enough to represent \emph{any} strict weak order on $\AlternativeSet$ because it can represent any real-valued set function $f\!:\!2^{\mathcal{F}}\!\rightarrow\!\mathbb{R}$ \cite{grabisch2000equivalent}, provided that $f(\emptyset)\!=\!0$.  
However, it requires to specify $2^n\!-\!1$ parameters. 
We therefore restrict our attention to additive models requiring fewer parameters.

\paragraph{The $\theta$-additive model} 
Given a set $\theta\!\subseteq\! 2^{\mathcal{F}}$, and a set function
$v\!:\!\theta\!\rightarrow\!\mathbb{R}$, we assume that $f$ is of the form $f(A)\!=\!\sum_{S \in \theta} I_A(S) v_S$, where $v_S$ stands again for $v(S)$. 
We call this the \emph{$\theta$-additive model}.
For this model, we may also use the notation $f_{\theta,v}(A)$ instead of $f(A)$.  
The 1-additive (resp. $k$-additive) model is the special case in which $\theta\!=\![\mathcal{F}]^1$ (resp. $\theta\!=\![\mathcal{F}]^k$).%

\begin{example}\label{ex : intro theta model}
Let $\mathcal{F}\!=\! \{a_1,a_2,a_3,a_4\}$ be a set of 4 elements, $\AlternativeSet\!=\! \{0,1\}^{4}$ and the DM's preferences be the strict weak order $\succsim$ given by :
\begin{align*}
     & & \{a_2, a_3, a_4\} & \succ\!\!\!  & \{a_1, a_3, a_4\} & \succ\!\!\! & \{a_1, a_2, a_4\}& \succ\!\!\! & \{a_3, a_4\} \\ & \succ\!\!\! & \{a_2, a_4\} & \succ\!\!\! &\{a_2, a_3\}& \succ\!\!\! & \{a_1, a_4\} & \succ\!\!\! &  \{a_1, a_3\} \\ & \succ\!\!\! & \{a_1, a_2\} & \succ\!\!\! & \{a_4\} & \succ\!\!\! & \{a_3\} & \succ\!\!\! & \{a_2\} \\ & \succ\!\!\! & \{a_1\} & \succ\!\!\! & A\!=\!\{a_1, a_2, a_3, a_4\} & \sim\!\!\! & \emptyset & \succ\!\!\! & B\!=\!\{a_1, a_2, a_3\}.
\end{align*}

These preferences can be explained by a clear negative synergy when $a_1$, $a_2$, and $a_3$ are chosen together (in $A$ and $B$). 
Interestingly, instead of using a complete 3-additive model, which would require the definition of 14 parameters, this strict weak order can be obtained by using a $\theta$-additive model with $\theta\!=\!\{\{a_1\},\{a_2\},\{a_3\},\{a_4\},\{a_1,a_2,a_3\}\}$ and $ v_{\{a_1\}}\!=\! 1$, $v_{\{a_2\}}\!=\!2$, $v_{\{a_3\}}\!=\! 3$, $v_{\{a_4\}}\!=\!4$, $v_{\{a_1,a_2,a_3\}}\!=\!-10$. This allows us to benefit from the expressiveness offered by 3-additivity while restricting the number of parameters. %

\end{example}

\subsection{The $\theta$-ordinal dominance relation}
In our elicitation setting, we assume that we have only access to a partial set $R$ of strict pairwise preferences provided by the DM.
This set may contain only a few comparisons. Our aim is to use these observed preferences to infer other strict pairwise preferences on the set of alternatives.
We formalize $R$ as a set of pairs $(A, B)\!\in\!\AlternativeSet^2$ such that $(A,B)\!\in\!R \Leftrightarrow A\succ\!B$.

Moreover, given $\theta$, the set of value functions on $\theta$ that are compatible with the preferences observed in $R$ is denoted by $V_{\theta}^R$:   
$$
V_{\theta}^R = \{v:\theta \rightarrow \mathbb{R}\,|\, \forall (A,B)\!\in\!R,\, f_{\theta,v}(A) > f_{\theta,v}(B)\}.
$$
Note that, for a given $\theta$, this set $V_{\theta}^R$ can be either empty or composed of an infinity of possible value functions on $\theta$. 
Notably, if this set is empty then the preferences of the user cannot be represented by a $\theta$-additive function. We denote by $\Theta^{R}$ the set $\{\theta \,|\, V_{\theta}^R\!\neq\!\emptyset \}$, i.e., the $\theta$'s such that the preferences in $R$ are consistent with a $\theta$-additive function.

Unfortunately, given $\theta\!\in\! \Theta^{R}$ such that $V_{\theta}^R\!\neq\!\emptyset$,
a pair $\{v,v'\}$ of value functions in $V_{\theta}^R$ may lead to infer opposite preferences, as illustrated below.
\begin{example}\label{ex : intro theta model 2}
Let $\mathcal{F} = \{a_1,a_2,a_3,a_4\}$. Let us assume that, contrary to Example~\ref{ex : intro theta model}, we now only observe preferences on the singletons $\{a_1\},\{a_2\},\{a_3\},\{a_4\}$: 
\begin{align*}
    & \{a_4\} \succ \{a_3\} \succ \{a_2\} \succ \{a_1\}, \mbox{or equivalently:}\\
     & R=\{(\{a_4\},\{a_3\}),\,(\{a_4\},\{a_2\}),\,(\{a_4\},\{a_1\}),\\
     & ~~~~~~~~~~~~~~~~~~~~~~~~(\{a_3\},\{a_2\}),\,(\{a_3\},\{a_1\}),\,(\{a_2\},\{a_1\})\}. 
\end{align*}
The two additive functions $v$ and $v'$ defined by: 
\begin{align*}
& v(\{a_1\})\!=\!1,\,v(\{a_2\})\!=\!2,\,v(\{a_3\})\!=\!3, \,v(\{a_4\})\!=\!5\\
\mbox{and } & v'(\{a_1\})\!=\!1,\,v'(\{a_2\})\!=\!3,\,v'(\{a_3\})\!=\!4,\,v'(\{a_4\})\!=\!5 
\end{align*}
are both in $V_{\theta}^R$, but we infer $\{a_1, a_4\}\succ \{a_2, a_3\}$ from $v$ while we infer $\{a_2, a_3\}\succ \{a_1, a_4\}$ from $v'$.
\end{example}

This example shows that, given $R$, choosing a specific function $v\!\in\!V_{\theta}^R$ can lead to infer  preferences that are only related to this arbitrary choice~\cite{bartee1971problem}. 
Our aim is to infer preferences for pairs outside $R$ in a reliable way by eliminating such arbitrary choices. In this purpose, we turn to a robust ordinal regression approach based on the observed preferences in $R$. %

Fishburn and Lavalle \cite{Fishburn1996BinaryIA} showed how one can obtain an \emph{ordinal dominance relation} from a partially specified 2-additive numerical model. We now explain how their idea can be extended to a $\theta$-additive model.

\begin{definition}
Let $\mathcal{F}$ be a set of elements, $\AlternativeSet \subseteq 2^{\mathcal{F}}$ a set of subsets  and $R$ a set of pairs $(A, B)\!\in\!\AlternativeSet^2$ where $(A,B) \in R \Leftrightarrow A \succ B$. Given $ \theta\!\in\!\Theta^{R}$, the $\theta$-ordinal dominance relation, denoted by $\succ_{\theta}^{R}$, is defined for $A,B \in \AlternativeSet$ by: 
$$
A \succ_{\theta}^R B \Leftrightarrow \forall v \in V_{\theta}^R,\, f_{\theta, v}(A) > f_{\theta, v}(B).
$$
\end{definition}

The $\theta$-ordinal dominance relation is independent from the choice of a specific $v\!\in\! V_{\theta}^R$. 
Naturally, $(A,B)\!\in\!R\!\Rightarrow\!A\!\succ_{\theta}^R\!B$.
Nevertheless, note that the binary relation $\succ_{\theta}^R$ is obviously partial, and we define the incomparability relation $\inc_{\theta}^R$ as: 
$$
A \inc_\theta^R B \Leftrightarrow 
\exists v , v' \in V_{\theta}^R,\,
f_{\theta, v}(A) \ge f_{\theta, v}(B) \mbox{ and } f_{\theta, v'}(B) \ge f_{\theta, v'}(A).
$$
If $A\!\succ_\Theta^R\!B$ then we can predict, based on $R$ and for a $\theta$-additive model, that $A$ is strictly preferred to $B$. If $A\!\inc_{\theta}^R\!B$ then no prediction is made %

We conclude this section by mentioning some properties of $\succ_\theta^R$: 
\begin{itemize}[noitemsep,topsep=2pt]
\item Unlike $\succ$, the relation $\succ_\theta^R$ is not a strict  weak order: it is asymmetric but it may not be complete nor negatively-transitive. The absence of preference prediction may occur in two situations that are not equivalent: either $A$ and $B$ belong to the same incomparability class of the (unknown) strict weak order $\succ$ on $\mathcal{A}$, i.e., $A\!\sim\!B$, or there is not enough preferential information in $R$ to conclude that $A\succ\!B$ or $B\succ\!A$.
\item Since the ordinal dominance relation depends on the preference set $R$ and on the model $\theta$, the relation $\succ_\theta^R$ evolves when $\theta$ or $R$ are restricted or extended. In particular, if $\theta'\! \subseteq\!\theta$ then any prediction that is yielded using ordinal dominance with the model $\theta$ is also yielded using ordinal dominance with the model $\theta'$; thus, if $V^R_\theta\!\neq\!\emptyset$ and $V^R_{\theta'}\!\neq\!\emptyset$, then $\theta'$ appears as more appealing from a preference learning standpoint since it allows more predictions to be made. Furthermore, one could prefer $\theta'$ over $\theta$ because of the philosophical principle of parsimony \cite[e.g.][]{blumer1987occam}.%
\end{itemize}

A more formal and detailed description of the properties of $\succ_\theta^R$ can be found in the supplementary material (Appendix A).

\subsection{The robust ordinal dominance relation}
Note that the ordinal dominance relation is dependent on the choice of a specific set $\theta\!\in\!\Theta^R$.
However, as shown in the following example, there may be several $\theta$'s in $\Theta^{R}$.

\begin{example}
Assume that $R$ consists of all pairwise preferences resulting from $\succ$ in  Example~\ref{ex : intro theta model}. Setting $\theta\!=\!\{\{a_1\},\{a_2\},$ $\{a_3\},\{a_4\}\}$ yields then $V_{\theta}^R\!=\!\emptyset$. In contrast, setting $\theta_1 = \{\{a_1\},\{a_2\},\{a_3\},\{a_4\},$ $\{a_1,a_2,a_3\}\}$ yields $V^{R}_{\theta_1}\neq \emptyset$. Actually, there are many other sets $\theta$ compatible with the preferences in $R$: it can be shown\footnote{It has been computer tested by brute force enumeration.} that $\Theta^{R}\!=\!\{\theta : \theta_1 \subseteq \theta \}$ for this example.
\end{example}

The question that naturally arises is whether we could find two different models $\theta_1, \theta_2 \in \Theta^R$ that are both compatible with the observed preferences in $R$ and such that $A\succ_{\theta_1}^R B$ and $B \succ_{\theta_2}^R A$ for a pair of alternatives $(A,B)\!\in\!\AlternativeSet^2$. Unfortunately, this situation may indeed happen:

\begin{example}
Let $R\!=\!\{(\{a_1\},\{a_2\})\}$, $\theta_1\!=\!\{\{a_1\}\}$ and $\theta_2\!=\!\{\{a_2\}\}$.  
Note that both $\theta_1$ and $\theta_2$ belong to $\Theta^R$. 
If we consider $\theta_1\!=\!\{\{a_1\}\}$, the set $V^R_{\theta_1}$ is compounded of value functions $v$ defined on $\theta_1$ such that $v(\{a_1\})\!>\!0$. Hence, for all $v\!\in\!V^R_{\theta_1}$ we have $f_{\theta_1, v}(\{a_1,a_2\})\!=\!v(\{a_1\})\!>\!0\!=\!f_{\theta_1, v}(\emptyset)$ and thus $\{a_1,a_2\}\!\succ_{\theta_1}^R\! \emptyset$. 
Conversely, if we consider $\theta_2\!=\!\{a_2\}$, the set $V^R_{\theta_2}$ is compounded of value functions $v$ defined on $\theta_2$ such that $v(\{a_2\})\!<\!0$.
This yields $f_{\theta_2, v}(\{a_1,a_2\})\!=\!v(\{a_2\})\!<\!0$ for each $v\!\in\!V^R_{\theta_2}$ and thus  $\emptyset\!\succ_{\theta_2}^R\!\{a_1,a_2\}$. 
\end{example}

In what follows, we will define a more robust variant of the ordinal dominance relation. 
This variant will take into account the plurality of models compatible with the observed preferences.  

Note that there always exists a $\theta$ able to represent $R$ (at worst, $\theta\!=\!2^\mathcal{F}$) and that if a $\theta$-additive model is compatible with $R$, then any $\theta'$-additive model with $\theta \subseteq \theta'$ is also compatible with $R$. 
For this reason, the number of sets $\theta$ compatible with the observed preferences may be very large.  

For this reason, we start by restricting the set of models to take into account. 
In this purpose, we need a binary relation $\sqsubseteq$ on $\Theta^R$, such that $\theta \sqsubseteq \theta'$ if $\theta$ is considered simpler than $\theta'$.
Our idea is to only consider sets $\theta$ that are minimal according to such a binary relation, i.e., $\theta$ such that $\not\exists \theta'\!\in\!\Theta^R$ for which $\theta'\!\sqsubseteq\!\theta$. 
This is motivated by the philosophical principle of parsimony that the simpler of two explanations is to be preferred (Occam's razor \cite{blumer1987occam}).
Different possible definitions for $\sqsubseteq$ will be discussed upon in the following subsection. 

We call $\sqsubseteq$-\textit{simplest} $\theta$ of $\Theta^R$ the parameter sets $\theta\!\in\!\Theta^{R}$ which are minimal w.r.t. $\sqsubseteq$, and we denote by $\Theta^R_{\sqsubseteq}$ their set.  
Based on $\Theta^R_{\sqsubseteq}$, we extend the ordinal dominance relation to define the \emph{$\sqsubseteq$-robust ordinal dominance relation}.
\begin{definition}
Let $\mathcal{F}$ be a set of elements, $\AlternativeSet \subseteq 2^F$ a set of subsets  and $R$ a set of pairs $(A, B)\!\in\!\AlternativeSet^2$ where $(A,B)\!\in\!R \Leftrightarrow A\!\succ\!B$. The $\sqsubseteq$-robust ordinal dominance relation, denoted by $\succ^R_{\sqsubseteq}$, is defined, for $A,B \in \AlternativeSet$, as follows: 
\begin{align*} 
   A \succ^R_{\sqsubseteq} B & \iff \forall \theta \in \Theta^{R}_{\sqsubseteq},\, A \succ_{\theta}^R B, \\
                & \iff \forall \theta \in \Theta^{R}_{\sqsubseteq},\, \forall v \in V_{\theta}^R,\, f_{\theta, v}(A) > f_{\theta, v}(B).
\end{align*}
\end{definition}
In other words, $A$ $\sqsubseteq$-robustly ordinally dominates $B$ if $A$ $\theta$-ordinally dominates $B$ according to all $\theta$ in $\Theta^R_{\sqsubseteq}$, i.e., all the  $\sqsubseteq$-\textit{simplest} $\theta$'s of $\Theta^R$.

\subsection{Different definitions for $\sqsubseteq$} \label{part: sqsubset_def}

We say that a relation $\sqsubseteq$ is \emph{based on} a function $\xi$ when $\theta \sqsubseteq \theta'$ if and only if $\xi(\theta)\leq\xi(\theta')$. Several aspects can be taken into account to define $\xi$:

\noindent$\bullet$ A first idea is to favor parameter sets $\theta$ that minimize the complexity of synergies between the attributes. To measure this complexity, we use the \emph{degree} of $\theta$, namely $\mathtt{deg}(\theta)\!=\!\max\{|S|:S\!\in\!\theta\}$ (i.e., the greatest cardinality of a subset of interacting attributes). This leads to the binary relation $\sqsubseteq_{\mathtt{deg}}$ based on $\mathtt{deg}$, i.e., $\theta_1 \sqsubseteq_{\mathtt{deg}} \theta_2 \Leftrightarrow \mathtt{deg}(\theta_1) \leq \mathtt{deg}(\theta_2)$.

\noindent$\bullet$ A second idea is to favor parameter sets $\theta$ having the \emph{sparsest} possible representation \cite{zhang2015survey}, i.e., those which minimize $\mathtt{card}(\theta) = |\theta|$. 
This choice yields the binary relation $\sqsubseteq_{\mathtt{card}}$, which is the relation based on the function $\mathtt{card}$, i.e., $\theta_1 \sqsubseteq_{\mathtt{card}} \theta_2 \Leftrightarrow \mathtt{card}(\theta_1) \leq \mathtt{card}(\theta_2)$.

\noindent$\bullet$ Alternatively, we define a binary relation combining the ideas of $\sqsubseteq_{\mathtt{deg}}$ and $\sqsubseteq_{\mathtt{card}}$ by considering both the number and the size of elements in a parameter set $\theta$. 
In this purpose, we define $\sqsubseteq_{\mathtt{ws}}$, the relation based on the function $\mathtt{ws}(\theta)\!=\!\sum_{S \in \theta} |S|$, i.e., $\theta_1 \sqsubseteq_{\mathtt{ws}} \theta_2 \Leftrightarrow \mathtt{ws}(\theta_1) \leq \mathtt{ws}(\theta_2)$.

\noindent$\bullet$ Lastly, we define the binary relation $\sqsubseteq_{\mathtt{lex}}$, defined by using lexicographically the binary relations $\sqsubseteq_{\mathtt{deg}}$, $\sqsubseteq_{\mathtt{card}}$, and $\sqsubseteq_{\mathtt{ws}}$, in this order. This relation could be seen as based on the function $\mathtt{lex}$ where $\mathtt{lex}(\theta) = n 4^n \mathtt{deg}(\theta) + n 2^n \mathtt{card}(\theta) + \mathtt{ws}(\theta)$. 

\begin{example}
Let $R\!=\! \{(\{a_1, a_2\},\{a_3, a_4\}),\,(\{a_1, a_2\},\{a_1, a_3\})\}$.
It is easy to see that $V_\theta^R \!\neq\!\emptyset$ for $\theta\!=\!\{\{a_1,a_2\}\}$, which corresponds to a model of degree 2. 
However, we may prefer being consistent with a model of degree 1, even if there are more elements in it: $\theta'=\{\{a_1\}, \{a_2\}\}$ or $\theta''=\{\{a_1\}, \{a_3\}\}$ or $\theta'''=\{\{a_2\} \}$. In this example, the minimal parameter set $\theta$ among $\theta',\theta'',\theta'''$ w.r.t. relation $\sqsubseteq_{\mathtt{deg}}$ (resp. $\sqsubseteq_{\mathtt{card}}$, $\sqsubseteq_{\mathtt{ws}}$, $\sqsubseteq_{\mathtt{lex}}$) is $\{\theta', \theta'', \theta'''\}$ (resp. $\{\theta'''\}$ in the three cases). 
\end{example}

\section{Preference prediction by using robust ordinal dominance}%
\label{sec:PrefPrediction}

Given a set $R$ of pairwise preferences %
and a binary relation $\sqsubseteq$ on $\Theta^R$, the preference learning method we propose consists in predicting that a subset $A$ is preferred to $B$ if $A\!\succ^R_{\sqsubseteq}\!B$, i.e., $A$ is preferred to $B$ for all simplest models $\theta\!\in\!\Theta^R$ and value functions $v\!\in\!V^R_\theta$. The purpose of this section is to detail the procedure for determining whether $A\!\succ^R_{\sqsubseteq_\mathtt{lex}}\!B$. It is organized as follows:

\noindent$\bullet$ We show that determining if $A\!\succ^R_\theta\!B$ is polytime in $|R|$ and $|\theta|$, while determining if $A\!\succ^R_{\sqsubseteq}\!B$ amounts to testing whether $\Theta^R_{\sqsubseteq}\cap \Theta^R_{B\succsim A}\!=\!\emptyset$, where $\Theta^R_{B\succsim A}\!=\!\{\theta \in \Theta^R : B\!\succ^R_{\theta}\!A \mbox{ or } B\!\sim^R_{\theta}\!A\}$ (Subsection~\ref{sec:DeterminingThatADominatesB}). %

\noindent$\bullet$  As determining an explicit representation of $\Theta^R_{\sqsubseteq}$ is likely to be cumbersome (as the size of $\Theta^R_{\sqsubseteq}$ can be very large), we turn to an implicit representation based on the values $\mathtt{deg}(\theta)$, $\mathtt{card}(\theta)$, $\mathtt{ws}(\theta)$ for $\theta\!\in\!\Theta^R_{\sqsubseteq}$. We thus study the computational complexity of determining $\mathtt{deg}(\theta)$ (resp. $\mathtt{card}(\theta)$, $\mathtt{ws}(\theta)$, $\mathtt{lex}(\theta)$) for $\theta\!\in\! \Theta^R_{\sqsubseteq}$ and $\sqsubseteq = \sqsubseteq_{\mathtt{deg}}$ (resp. $\sqsubseteq = \sqsubseteq_{\mathtt{card}}$, $\sqsubseteq = \sqsubseteq_{\mathtt{ws}}$, $\sqsubseteq = \sqsubseteq_{\mathtt{lex}}$), showing that the former problem can be solved in polynomial time, while the others are NP-hard (Subsection~\ref{sec:complexity}).

\noindent$\bullet$ The implicit representation of $\Theta^R_{\sqsubseteq_{\mathtt{lex}}}$ is based on the following idea: if we know that $\theta_0\!\in\!\Theta^R_{\sqsubseteq_{\mathtt{lex}}}$, then $\theta\!\in\!\Theta^R_{\sqsubseteq_{\mathtt{lex}}} \Leftrightarrow (\mathtt{deg}(\theta), \mathtt{card}(\theta),\mathtt{ws}(\theta))\!=\!(\mathtt{deg}(\theta_0), \mathtt{card}(\theta_0),\mathtt{ws}(\theta_0))$. It is thus enough to determine a single model $\theta_0\!\in\!\Theta^R_{\sqsubseteq_{\mathtt{lex}}}$ to be able to determine whether a model belongs to $\Theta^R_{\sqsubseteq_{\mathtt{lex}}}$. This is why we propose a Mixed Integer Program (MIP) to compute a model $\theta\!\in\!\Theta^R_{\sqsubseteq_{\mathtt{lex}}}$, derived from a linear program for determining whether a model $\theta$ belongs to $\Theta^R$ (Subsection~\ref{sec:computeThetaMin}).

\noindent$\bullet$ We derive from it another MIP to compute a model in $\Theta^R_{\sqsubseteq_{\mathtt{lex}}}\cap \Theta^R_{B\succsim A}$, concluding $A\!\not\succ^R_{\sqsubseteq_{\mathtt{lex}}}\!B$ if it exists, $A\!\succ^R_{\sqsubseteq_{\mathtt{lex}}}\!B$ otherwise (Subsection~\ref{sec:DeterminingRobustOrdinalDominance}).

\subsection{Determining whether $A\!\succ^R_{\theta}\!B$ and whether $A\!\succ^R_{\sqsubseteq}\!B$}
\label{sec:DeterminingThatADominatesB}

We first show that, unsurprisingly, linear programming provides an operational tool for determining whether $A\!\succ^R_{\theta}\!B$. %
Viewing a value function on $\theta$ as a vector $v\!=\!(v_S)_{S\in\theta}$ where $v_S\!=\!v(S)$, the set $V_{\theta}^R$ corresponds to the polyhedron defined by the following linear constraints in the $|\theta|$-dimensional parameter space (where each parameter $v_S$ corresponds to a dimension)\footnote{The right hand side of the constraint is here set to 1, but it could be set to any strictly positive constant as utilities $v_S$ are always compatible with $R$ to within a positive multiplicative factor.}:
\begin{equation*}
\begin{split}
    \forall (X,Y) \in R, \sum_{S \in \theta} I_X(S) v_S - \sum_{S \in \theta} I_Y(S) v_S \ge 1. 
\end{split}
\end{equation*}

For a given set $R$ of strict pairwise preferences and a model $\theta\!\in\!\Theta^R$, checking whether $A\!\succ^R_\theta\!B$ can be evaluated in polynomial time in $|R|$ and in $|\theta|$ by solving the following linear program $\mathcal{P}_{A \succ^R_\theta B}$, where there is one variable $v_S\!\in\!\mathbb{R}$ for each pair $S\!\in\!\theta$: 
\begin{equation*} \label{eq:APrefThetaB}
    (\mathcal{P}_{A \succ^R_\theta B}) \left\{ 
\begin{array}{lr}
    \min \displaystyle\sum_{S \in \theta} I_A(S) v_S - \displaystyle\sum_{S \in \theta} I_B(S) v_S &\\
    \displaystyle\sum_{S \in \theta} (I_X(S) - I_Y(S)) v_S \ge 1 &\forall (X,Y) \in R\setminus\{(A,B)\},\\
    v_S \in \mathbb{R} &\forall S \in \theta.
\end{array}    
     \right. 
\end{equation*}
We have that $A\!\succ^R_\theta\!B$ if and only if the optimal value of $\mathcal{P}_{A \succ^R_\theta B}$ is strictly positive, as
it implies that $\sum_{S \in \theta} I_A(S) v_S > \sum_{S \in \theta} I_B(S)$ for all $v\!\in\!V^R_\theta$.

In contrast with this positive complexity result for ordinal dominance, determining whether $A\!\succ^R_{\sqsubseteq}\!B$ by direct use of the definition of robust ordinal dominance would require a high computational burden. We overcome this difficulty by reducing this problem to testing whether $\Theta^R_{\sqsubseteq}\cap \Theta^R_{B\succsim A}$ is empty. 

To achieve this reduction, let us study the relationships between $\Theta^R_{A\succsim B}$, $\Theta^R_{B\succsim A}$ and $\Theta^R_{\sqsubseteq}$. For visual support, the reader may refer to Figure~\ref{fig:ThetaR}. We recall that we denote by $\Theta^R_{B\succsim A}$ the set $\{\theta \in \Theta^R : B\!\succ^R_{\theta}\!A \mbox{ or } B\!\sim^R_{\theta}\!A\}$. As one of the relations $A\!\succ_\theta\!B$ or $B\!\succ_\theta\!A$ or $A\!\sim_\theta \!B$ holds for any $\theta\!\in\!\Theta^R$, we have that $\Theta^R\!=\!\Theta^R_{A\succsim B}\!\cup\!\Theta^R_{B\succsim A}$. Consequently, $\Theta^R_{\sqsubseteq}\!\subseteq\!\Theta^R_{A\succsim B}\!\cup\!\Theta^R_{B\succsim A}$ because $\Theta^R_{\sqsubseteq}\!\subseteq\!\Theta^R$. Furthermore, $\Theta^R_{A\succsim B}\!\cap\!\Theta^R_{B\succsim A}\!=\!\{\theta\!\in\!\Theta^R:A\!\sim_\theta\!B\}\!\neq\!\emptyset$ as soon as there exists $\theta\!\in\!\Theta^R$ for which $A \sim_\theta B$. 

\begin{figure}[!h]
\begin{center}
\scalebox{0.735}{\begin{tikzpicture}
  \draw[fill=none, pattern=north east lines] (0,0) circle (2cm);

  \draw[fill=none, pattern=dots] (3cm,0) circle (2cm);

  \begin{scope}
    \clip (0,0) circle (2cm);
    \fill[white] (3cm,0) circle (2cm);
  \end{scope}

  \begin{scope}
    \clip (1.5cm,0) circle (1cm);
    \fill[red, pattern=north east lines] (0,0) circle (2cm);
  \end{scope}

  \begin{scope}
    \clip (1.5cm,0) circle (1cm);
    \fill[red, pattern=dots] (3cm,0) circle (2cm);
  \end{scope}

  \begin{scope}
    \clip (0,0) circle (2cm);
    \fill[white] (3cm,0) circle (2cm);
 \end{scope}
  \draw[fill=none] (1.5cm,0) circle (1cm);
  \node at (1.5cm,0) {$\Theta_{\sqsubseteq}^{R}$};

  \node at (-3cm,0) {$\Theta^{R}_{A \succsim B}$};
  \node at (6cm,0) {$\Theta^{R}_{B \succsim A}$};
\end{tikzpicture}}
\caption{$(\Theta^R_{\sqsubseteq}\cap \Theta^R_{B\succsim A}\!=\!\emptyset \Leftrightarrow A \succ^R_{\sqsubseteq} B)$ and $(\Theta^R_{\sqsubseteq}\cap \Theta^R_{A\succsim B}\!=\!\emptyset \Leftrightarrow B \succ^R_{\sqsubseteq} A)$.}
\label{fig:ThetaR}
\end{center}
\end{figure}
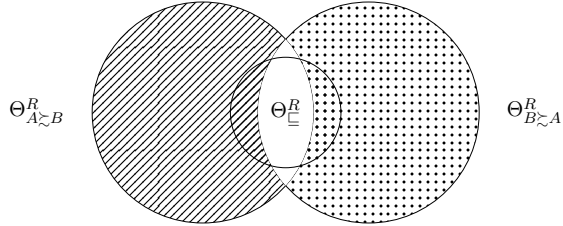

To evaluate whether a robust ordinal dominance relation holds between two subsets $A$ and $B$, we examine if one of the following conditions holds:%
\begin{enumerate}[noitemsep,topsep=2pt]
    \item[$(i)$] $\Theta^{R}_{\sqsubseteq} \cap \Theta^R_{B\succsim A}\!=\!\emptyset$,
    \item[$(ii)$] $\Theta^{R}_{\sqsubseteq} \cap \Theta^R_{A\succsim B}\!=\!\emptyset$. 
\end{enumerate}
We have indeed the following result:
\begin{restatable}{proposition}{propROD}\label{propROD}
For any $A,B \! \subseteq \mathcal{F}$, we have $A\succ_{\sqsubseteq}^{R} B \Leftrightarrow \Theta^{R}_{\sqsubseteq} \cap \Theta^{R}_{B\succsim A} = \emptyset.$
\end{restatable}

\noindent \emph{Proof.} It follows from the following sequence of equivalences:
\begin{equation*}
\begin{split}
A \succ^R_{\sqsubseteq} B \Leftrightarrow \forall \theta \in \Theta^R_{\sqsubseteq}, A \succ^R_\theta B & \Leftrightarrow \forall \theta \in \Theta^R_{\sqsubseteq}, B \not\succ^R_\theta A \mbox{ and } A \not\sim^R_\theta B\\  & \Leftrightarrow \Theta^{R}_{\sqsubseteq} \cap \Theta^{R}_{B\succsim A} = \emptyset.  ~~~~~~~~~~~~~~~~~~~~~~~~~~~~~~~\square
\end{split} 
\end{equation*}

\medskip

Symmetrically, we have obviously that $B\succ_{\Theta}^{R} A \Leftrightarrow \Theta^{R}_{\sqsubseteq} \cap \Theta^{R}_{A\succsim B} = \emptyset$. To test whether $\Theta^{R}_{\sqsubseteq} \cap \Theta^R_{B\succsim A}\!=\!\emptyset$, the mathematical programming approach we propose applies to cases where relation $\sqsubseteq$ is based on a function $\xi$. 
The approach starts by computing a \emph{single} model $\theta\!\in\!\Theta^{R}$ minimizing $\xi(\theta)$, which is enough for determining the value $\xi(\theta)$ of \emph{any} $\theta\!\in\!\Theta^{R}_{\sqsubseteq}$, as they all share the same optimal value $\xi(\theta)$. We now study the complexity of computing such an optimal $\theta$ in $\Theta^R$. More precisely, we study the complexity of the following decision problem MIN-$\theta$-$\xi$, for $\xi\!\in\!\{\mathtt{card},\mathtt{ws},\mathtt{deg}, \mathtt{lex}\}$ (as is well-known, the optimization problem is at least as hard as its decision variant):
\vspace{-0.5cm}
\begin{cproblem}{MIN-$\theta$-$\xi$}
\textbf{INPUT:} A set $\AlternativeSet$ of alternatives, a set $R = \{(A,B), A,B \in \AlternativeSet\}$ of strict pairwise preferences, an integer $\tau \in \mathbb{Z}^+$. \\
\textbf{QUESTION:} Does there exist $\theta \in \Theta^{R}$ such that $\xi(\theta) \le \tau$?
\end{cproblem}

\subsection{Computational complexity of MIN-$\theta$-$\xi$ for $\xi\!\in\!\{\mathtt{card},\mathtt{ws},\mathtt{lex},\mathtt{deg}\}$}
\label{sec:complexity}

We show here that MIN-$\theta$-$\xi$ is NP-hard for $\sqsubseteq\, \in\!\{\mathtt{ws},\mathtt{card},\mathtt{lex}\}$, while it can be solved in polynomial time for $\sqsubseteq\, =\! \mathtt{deg}$.

\begin{theorem}
MIN-$\theta$-$\mathtt{card}$ and MIN-$\theta$-$\mathtt{ws}$ are NP-complete.
\end{theorem}
\begin{proof}
The membership of MIN-$\theta$-$\mathtt{card}$ to NP follows from the fact that $\min_{\theta} \texttt{card}(\theta)\!\leq\!2|R|$ and checking that $\theta\!\in\!\Theta^R$ can be done in polynomial time in $|R|$ and $|\theta|$. Indeed, the parameter set $\theta\!=\!\{A\!\in\!\mathcal{A}:(A,\cdot)\!\in\!R \mbox{ or } (\cdot,A)\!\in\!R\}$ obviously belongs to $\Theta^R$, and $|\theta|\!\leq\!2|R|$. The proof that MIN-$\theta$-$\mathtt{ws}$ belongs to NP is similar, based on the fact that
$\min_{\theta} \texttt{ws}(\theta)\!\leq\!2|R|\times n$.

To prove the NP-hardness, we use a reduction from Hitting Set: %
\vspace{-0.5cm}
\begin{cproblem}{Hitting Set}
    \textbf{INPUT:} Given a set of $n$ elements: $\mathcal{X} = \{x_i\}_{1 \le i \le n}$, a family of $m$ sets $\mathcal{S} = \{S_i : S_i \subseteq \mathcal{X}, 1 \le i \le m\}$, and an integer $\tau \in \mathbb{Z}^+$.\\
    \textbf{QUESTION:} Does there exist $\mathcal{X}' \subseteq \mathcal{X}$ such that $\forall S_i \in \mathcal{S}, S_i\cap \mathcal{X}'  \neq \emptyset$ and $|\mathcal{X}'|\leq \tau$?
\end{cproblem}
Given an instance $(\mathcal{X}, \mathcal{S}, \tau)$ of the Hitting Set problem, we define the following instance $(\AlternativeSet, R, \tau')$ of MIN-$\theta$-$\mathtt{card}$ (resp. MIN-$\theta$-$\mathtt{ws}$).\\
We let $\AlternativeSet = \mathcal{S} \cup \{\emptyset\}$, $\tau' = \tau$, and consider the following set of preferences:
$$
R = \{(S,\emptyset) : S \in \mathcal{S}\}.
$$
Now we show that $(\mathcal{X}, \mathcal{S}, \tau)$ is a yes-instance of Hitting Set iff  $(\AlternativeSet, R, \tau')$ is a yes-instance of MIN-$\theta$-$\mathtt{card}$ (resp. MIN-$\theta$-$\mathtt{ws}$).
Note that a set $\theta$ belongs to $\Theta^R$ if and only if it satisfies the following condition: 
\[
     \forall (S, \emptyset)\!\in\!R,\, \exists\,T\!\in\! \theta \text{ such that } T \subseteq S. \tag{C}
\]
Indeed, each preferences in $R$ can then be satisfied by assigning positive values to parameters entailed by the elements of $\theta$. 
Moreover, note that if a set $\theta$ satisfies C and $\exists\,T\!\in\!\theta$ such that $|T|>1$, then the set $\theta'$ obtained from $\theta$ by replacing $T$ by any singleton $\{x\}\!\subset\!T$ also satisfies C. 
Hence, within the sets satisfying C and minimizing $\mathtt{card}$, there exists a set $\theta'$ compounded only of singletons, minimizing both $\mathtt{card}$ and $\mathtt{ws}$ (because $\mathtt{card}(\theta)\!=\!\mathtt{ws}(\theta)$ if $\theta$ is compounded only of singletons). By taking $\mathcal{X'}\!=\!\{x:\{x\} \in \theta'\}$, we obtain a hitting set of size $|\mathcal{X'}|\le \tau$.
This yields the following conclusion: there exists a hitting set of size $s \le \tau$ if and only if there exists a set $\theta$ satisfying C such that $\mathtt{card}(\theta)\!=\!s$ (resp. $\mathtt{ws}(\theta) \!=\!s$). This argument completes the proof.
\end{proof}

The following result is a direct consequence of the previous one:

\begin{corollary}\label{Corol-lexNPcomplete}
    MIN-$\theta$-$\mathtt{lex}$ is NP-hard.
\end{corollary}
\begin{proof}
Given an instance $(\AlternativeSet,R,\tau)$ of the MIN-$\theta$-$\mathtt{card}$ problem, we could solve for each degree $d \in \{0, 1, \ldots |\mathcal{F}|\}$ an instance $(\AlternativeSet,R,\tau')$ of the MIN-$\theta$-$\mathtt{lex}$ problem where $\tau'= dn4^n + (\tau + 1) n2^n$. 
\end{proof}

In contrast, we show a polynomial-time complexity result for MIN-$\theta$-$\mathtt{deg}$, by resorting to the \emph{kernel trick}, widely used in machine learning \cite[see e.g.,][]{scholkopf2018learning}. Given a vector space $\mathcal{X}$ of dimension $n_{\mathcal{X}}$ and a transformation function $\varphi:\mathcal{X}\rightarrow \mathcal{Y}$, where the dimension $n_{\mathcal{Y}}$ of vector space $\mathcal{Y}$ is exponential in $n_{\mathcal{X}}$, the kernel trick consists in computing the scalar products $\langle \varphi(x),\varphi(y) \rangle$ of $x,y\!\in\!\mathcal{X}$ in polynomial time in $n_\mathcal{X}$, by using a kernel function $K(x,y)$ that returns the value $\langle \varphi(x),\varphi(y) \rangle$ without requiring to explicit $\varphi(x)$ and $\varphi(y)$. In our setting, $\mathcal{X}$ is the set of characteristic vectors of subsets $A$ of $\mathcal{F}$, and $\mathcal{Y}$ the set of ``augmented'' characteristic vectors containing additional dimensions corresponding to binary values $I_A(S)$ for $S \in [\mathcal{F}]^\tau$ (more details in the proof). The complexity result is formulated as follows:

\begin{theorem}\label{th: deg}
    MIN-$\theta$-$\mathtt{deg}$ can be solved in polynomial time in $|R|$ and $n$.
\end{theorem}
\begin{proof}
Let $(\mathcal{A},R,\tau)$ be an instance of MIN-$\theta$-$\mathtt{deg}$. 
We wish to determine if preferences in $R$ can be represented by a $\theta$-additive model with $\theta\!=\![\mathcal{F}]^\tau$. 
For notational convenience, we set $\theta^{(\tau)}\!=\![\mathcal{F}]^\tau$ and $n_\tau\!=\!|\theta^{(\tau)}|\!=\! \sum_{i=1}^\tau {n \choose i}$. 
We associate to $\theta^{(\tau)}$ the vector $\overrightarrow{\theta^{(\tau)}}\!=\! (S_1,\ldots,S_{n_\tau})$, where subsets $S\!=\!\{a_{i_1},\ldots,a_{i_k}\}$ ($i_1\!<\! \ldots\!<\!i_k$) are indexed in lexicographic order of vectors $(|S|,i_1,\ldots,i_k)$. 

For instance, if  $\mathcal{F}=\{a_1,a_2,a_3\}$ and $\theta = \theta^{(3)}$ then 

$\overrightarrow{\theta} = (\{a_1\},\{a_2\},\{a_3\},$ $\{a_1,a_2\},\{a_1,a_3\},\{a_2,a_3\}, \{a_1,a_2,a_3\})$.  Additionally, for a value function $v:\theta\rightarrow\mathbb{R}$, we denote by $\overrightarrow{v} = (v_{S_1},\ldots,v_{S_{n_\tau}})$ the vector of values associated to the elements of $\overrightarrow{\theta^{(\tau)}}$ ordered in the same fashion.  
Finally, given $A \in \AlternativeSet$, we denote by $\overrightarrow{A}_{\tau}$ the binary vector $\overrightarrow{A_\tau} = (I_A(S_1),\ldots,I_A(S_{n_\tau}))$ where $I_A(S_i)$ is the indicator function of $S_i \in \theta^{(\tau)}$. 

Problem MIN-$\theta$-$\mathtt{deg}$ %
evaluates if the following proposition holds:
$$
\exists \overrightarrow{v} \in \mathbb{R}^{n_\tau} \text{ s.t. } \forall (A,B) \in R; \overrightarrow{A}_\tau \overrightarrow{v}^T > \overrightarrow{B}_\tau \overrightarrow{v}^T.
$$
A value vector $\overrightarrow{v}$ of minimum norm can be determined by solving the following convex quadratic program:
\begin{align*}
& \min_{\overrightarrow{v} \in \mathbb{R}^{n_\tau}} &  \frac{1}{2} \overrightarrow{v} \overrightarrow{v}^T & \\
& \mbox{s.t.} & \overrightarrow{A}_\tau\overrightarrow{v}^T \ge \overrightarrow{B}_\tau \overrightarrow{v}^T + 1& & \forall (A,B) \in R 
\end{align*}
Using the same trick as \citeauthor{domshlak2005unstructuring} \cite{domshlak2005unstructuring}, instead of solving this program whose number $n_\tau$ of variables is not polynomial in the size of our instance of MIN-$\theta$-$\mathtt{deg}$ (because $\tau$ is an input variable and not a constant), we consider its Wolfe dual defined by:
\begin{align*}
& & \max_{\alpha \in \mathbb{R}^{|R|}} &  \sum_{(A,B) \in R}\!\!\!\alpha_{(A,B)}  - \frac{1}{2} \sum_{(A,B) \in R} \sum_{(C,D) \in R} \!\!\!\alpha_{(A,B)} \alpha_{(C,D)} (\overrightarrow{A}_\tau - \overrightarrow{B}_\tau)(\overrightarrow{C}_\tau - \overrightarrow{D}_\tau)^T & \\
& \mbox{s.t.} &  & \alpha \geq 0 & & 
\end{align*}

By defining the kernel function
$
K^{(\tau)}(A, B) = \overrightarrow{A}_\tau {\overrightarrow{B}_{\tau}}^{T},
$
the previous program can be written as:
\begin{align*}
\max_{\alpha \in (\mathbb{R^+})^{|R|}} &  & \sum_{(A,B) \in R} \alpha_{(A,B)} &  -\frac{1}{2} \sum_{(A,B) \in R} \sum_{(C,D) \in R} \alpha_{(A,B)} \alpha_{(C,D)} &  \\
 &  &  & (K^{(\tau)}(A,C) - K^{(\tau)}(A,D) - K^{(\tau)}(B,C) + K^{(\tau)}(B,D)) 
\end{align*}

which can be solved in polynomial time in $|R|$ and $n$ provided that $K^{(\tau)}(X, Y)$ can be evaluated in polynomial time in $n$ without expliciting $X$ and $Y$.

Indeed, since the reformulation yields a convex quadratic program of polynomial size in the input data, the problem can then be solved in polynomial time (by polynomial time solvability of convex quadratic programming \cite{kozlov1979polynomial,kozlov1980polynomial}). 
We now prove that $K^{(\tau)}(X,Y)$ can be efficiently computed without expliciting $X$ and $Y$.
Let $k$ be the size of the intersection between $X$ and $Y$, i.e., $k = |X \cap Y|$. 
Note that $K^{(\tau)}(X,Y)$ counts the number of parameters of $\theta^{(\tau)}$ that are subsets of both $X$ and $Y$. 
We conclude by noting that the number of such elements corresponds to $ \sum_{i=1}^\tau {k \choose i}$, i.e., the number ($<\!2^n$) of non-empty subsets of size less than or equal to $\tau$ in $X \cap Y$. 
\end{proof}

\begin{remark}
Note that \citeauthor{TehraniStrickertHullermeier2014} \cite{TehraniStrickertHullermeier2014} and \citeauthor{herin2023learning}
\cite{herin2023learning} have proposed kernel functions $K(x,y)$ that return the scalar product $\langle \varphi(x),\varphi(y) \rangle$ of augmented vectors $\varphi(x),\varphi(y)$ used to obtain an additive expression $\langle m,\varphi(x)\rangle$ of a discrete Choquet integral $C(x)$, where $m$ is the vector of M\"obius masses obtained from the capacity used in $C(x)$. It turns out that there is a close link between $f_{\theta,v}$ and a Choquet integral $C(x)$ expressed as $\langle m,\varphi(x)\rangle$ (note however that we do not impose the constraints on the $v(S)$ values ensuring the monotonicity of the capacity, or the normalization constraint $\sum_S v(S)\!=\!1$). However, their kernel functions do not use the same calculations as ours: we take advantage of the particular case we study, where all components of $x$ take binary values, to compute the kernel function in $O(n)$ instead of $O(n^2)$.
\end{remark}

Algorithm~\ref{alg:degree} takes as input a set $R$ of strict pairwise preferences and computes $\min\{\mathtt{deg}(\theta)\!:\!\theta\!\in\!\Theta^R\}$ by solving a sequence of convex quadratic programs establishing whether there exists $\theta\!\in\!\Theta^R$ such that $\xi(\theta)\!=\!\tau$ (which holds if the optimal value of the program is bounded). The variable $\tau$ is gradually incremented from 1. At each iteration, the objective function parameters are updated by using the kernel trick, which makes the procedure polynomial-time in $|R|$ and $n$. 

\begin{algorithm}
\caption{Compute $\min\{\texttt{deg}(\theta):\theta \in \Theta^R\}$}
\label{alg:degree}
\begin{algorithmic}%
\footnotesize
\Require set $R$ of strict pairwise preferences
\Ensure $\min\{\texttt{deg}(\theta):\theta \in \Theta^R\}$ 
\State $\tau \leftarrow 1$
\For{$(A,B)\!\in\!R$}
    \For{$(C,D)\!\in\!R$}
\Comment{{\footnotesize Initialization of dictionary $Q$}}
    \State {\footnotesize $Q[A,B,C,D] \leftarrow |A \cap C|-|A \cap D| -|B \cap C|+|B \cap D|$}
    \EndFor
\EndFor
\While{{\footnotesize $\displaystyle\max_{\alpha\ge 0} \!\!\displaystyle\sum_{(A,B)\in R}\!\!\!\alpha_{(A,B)}-\frac{1}{2}\!\!\displaystyle\sum_{(A,B)\in R}\displaystyle\sum_{(C,D)\in R}\!\!\!\alpha_{(A,B)}\alpha_{(C,D)}Q[A,B,C,D]$ is unbounded}}
\Comment{{\footnotesize the $\alpha_{(X,Y)}$'s are the variables of the convex quadratic program}}\newline
\Comment{{\footnotesize$\alpha\!\ge\! 0$ means that $\alpha_{(X,Y)}\!\ge\! 0$ for all $(X,Y)\!\in\!R$}}\newline
\Comment{{\footnotesize$Q$ contains the coefficients of the objective function, updated at each iteration}}
\State $\tau \leftarrow \tau + 1$
    \For{$(A,B)\!\in\!R$}
    \For{$(C,D)\!\in\!R$}
    \State {\footnotesize $Q[A,B,C,D] \leftarrow Q[A,B,C,D]+{|A \cap C| \choose \tau}-{|A \cap D| \choose \tau}-{|B \cap C| \choose \tau}+{|B \cap D| \choose \tau}$}
    \EndFor
    \EndFor
    \EndWhile
    \State \Return $\tau$
\end{algorithmic}
\end{algorithm}

\subsection{Computing $(\mathtt{deg}(\theta), \mathtt{card}(\theta),\mathtt{ws}(\theta))$ for $\theta\!\in\!\Theta^R_{\sqsubseteq_{\mathtt{lex}}}$\label{sec:computeThetaMin}}%

As all models $\theta\!\in\!\Theta^R_{\sqsubseteq_{\mathtt{lex}}}$ share the same vector $(\mathtt{deg}(\theta), \mathtt{card}(\theta),\mathtt{ws}(\theta))$, it is enough to compute a single model $\theta\!\in\!\Theta^R_{\sqsubseteq_{\mathtt{lex}}}$ to deduce this vector, which will be required to determine whether $A\!\succ^R_{\sqsubseteq_{\mathtt{lex}}}\!B$.
The negative complexity result (Corollary~\ref{Corol-lexNPcomplete}) regarding the computation of a model $\theta\!\in\!\Theta^R_{\sqsubseteq_{\mathtt{lex}}}$ does not prevent us from proposing an exact solution method that will prove efficient in practice. For this purpose, we first present a Linear Program (LP) allowing us to determine in polynomial time in $|R|$ and $|\theta|$ whether $\theta\!\in\!\Theta^R$, given a model $\theta$ and a set $R$ of strict pairwise preferences. From this LP, we will then develop a MIP formulation for computing $\theta\!\in\!\Theta^R_{\sqsubseteq_{\mathtt{lex}}}$.

For a given set $R$ of strict pairwise preferences and a given model $\theta$, checking whether $\theta\!\in\!\Theta^R$ can be evaluated in polynomial time in $|R|$ and in $|\theta|$ by solving the following linear program $\mathcal{P}_{\theta}$, where there is one variable $e_{(A,B)}\!\ge\!0$ for each pair $(A,B)$ in $R$: 
\begin{equation*} \label{eq : primal}
    (\mathcal{P}^R_{\theta}) \left\{ 
\begin{array}{lr}
    \min \displaystyle\sum_{(A,B) \in R } e_{(A,B)} &\\
    \displaystyle\sum_{S \in \theta} (I_A(S) - I_B(S)) v_S + e_{(A,B)} \ge 1 &\forall (A,B) \in R,\\
    e_{(A,B)} \ge 0 &\forall (A,B) \in R,\\
    v_S \in \mathbb{R} &\forall S \in \theta.
\end{array}    
     \right. 
\end{equation*}
We have that $\theta\!\in\!\Theta^R$ if and only if the optimal value of $\mathcal{P}^R_{\theta}$ is 0, because 
in this case we can find values for variables $v_S$ that respect all the preferences in $R$ without the help of the additional slack variables $e_{(A,B)}$.

We now show how to derive, from $\mathcal{P}^R_{\theta}$, a MIP formulation for computing a model $\theta\!\in\!\Theta^R_{\sqsubseteq_\mathtt{lex}}$. For this, we first compute $\mathtt{deg}(R)\!=\!\min\{\mathtt{deg}(\theta):\theta\!\in\!\Theta^R\}$, by using Algorithm~\ref{alg:degree}. We then add a  binary variable $b_S$ for each $S\!\in\! [\mathcal{F}]^{\mathtt{deg}(\theta)}$, as well as big-M constraints to ensure that $b_S\!=\!1$ iff $S\!\in\!\theta$ (i.e., $S$ belongs to the model $\theta\!\in\!\Theta^R_{\sqsubseteq_\mathtt{lex}}$).
Determining a model $\theta\!\in\!\Theta^{R}_{\sqsubseteq_{\mathtt{lex}}}$ can be done by solving the following lexicographic optimization problem: %
\begin{equation*} 
    (\mathcal{P}^R_{\sqsubseteq_\mathtt{lex}}) \left\{ 
\begin{array}{lrr}
\min\displaystyle\lex \sum_{S \in [\mathcal{F}]^{\mathtt{deg}(R)}} b_S,\, \sum_{S\in[\mathcal{F}]^{\mathtt{deg}(R)}} b_S |S| & \notag \\
    \displaystyle\sum_{S \in [\mathcal{F}]^{\mathtt{deg}(R)}} (I_A(S) - I_B(S)) v_S \ge 1 &\forall (A,B) \in R, & \constraintlabel{1}\\
    -b_SM \le v_S \le b_SM & \forall S \in [\mathcal{F}]^{\mathtt{deg}(R)}, &\\
    b_S \in \{0,1\} & \forall S \in [\mathcal{F}]^{\mathtt{deg}(R)}. &
\end{array}
\right.
\end{equation*}
\noindent where $M\!=\!(2\sum_{i=1}^{\mathtt{deg}(R)}{n \choose i} + |R|)\times (|R|)^{2|R|+2}$, so that if the values $v_S$ can be set to satisfy constraints 1, then there exist such values in the interval $[-M,M]$~(see \cite{papadimitriou1981complexity}).
Every feasible instantiation of variables $v_S,b_S$ in $\mathcal{P}^R_{\sqsubseteq_\mathtt{lex}}$ corresponds to an element $\theta\!\in\!\Theta^R$, namely $\theta\!=\!\{S\!\in\![\mathcal{F}]^{\mathtt{deg}(R)}:b_S\!=\!1\}$. Lexicographic optimization amounts to determine, among feasible instantiations of $v_S,b_S$ that minimize the first objective $\sum_{S \in [\mathcal{F}]^{\mathtt{deg}(R)}} b_S$, one that minimizes the second objective $\sum_{S\in[\mathcal{F}]^{\mathtt{deg}(R)}} b_S |S|$. It is well-known that this can be achieved as follows, using a mixed integer programming solver: 
\begin{itemize}[noitemsep,topsep=2pt]
    \item first, we solve the MIP $\mathcal{P}_1$ obtained by replacing the lexicographic objective function in $\mathcal{P}^R_{\sqsubseteq_\mathtt{lex}}$ by $\min\sum_{S \in [\mathcal{F}]^{\mathtt{deg}(R)}} b_S$; 
    \item denoting by $\opt_1$ the optimal value of $\mathcal{P}_1$, we then solve the MIP $\mathcal{P}_2$ where the objective function in $\mathcal{P}^R_{\sqsubseteq_\mathtt{lex}}$ is replaced by $\min\sum_{S\in[\mathcal{F}]^{\mathtt{deg}(R)}} b_S |S|$, under the additional constraint $\sum_{S \in [\mathcal{F}]^{\mathtt{deg}(R)}} b_S \le \opt_1$.
\end{itemize}
As every feasible instantiation corresponds to a model $\theta$ of minimal degree $\mathtt{deg}(\theta)$ (i.e., $\mathtt{deg}(\theta)\!=\!\mathtt{deg}(R)$), we thus obtain a model $\theta\!\in\!\Theta^R_{\sqsubseteq_{\mathtt{\mathtt{lex}}}}$, from which we deduce $(\mathtt{deg}(\theta), \mathtt{card}(\theta),\mathtt{ws}(\theta))$ for $\theta\!\in\!\Theta^R_{\sqsubseteq_{\mathtt{lex}}}$. In the following, we denote by $(\mathtt{deg}_\mathtt{lex}, \mathtt{card}_\mathtt{lex},\mathtt{ws}_\mathtt{lex})$ the vector $(\mathtt{deg}(\theta), \mathtt{card}(\theta),\mathtt{ws}(\theta))$ for $\theta\!\in\!\Theta^R_{\sqsubseteq_{\mathtt{lex}}}$.

\subsection{Determining whether $A\!\succ^R_{\sqsubseteq_{\mathtt{lex}}}\!B$ \label{sec:DeterminingRobustOrdinalDominance}}

Determining whether $A\!\succ^R_{\sqsubseteq_{\mathtt{lex}}}\!B$ amounts to solve:
\begin{equation*} 
(\mathcal{P}_{A\succ^R_{\sqsubseteq_{\mathtt{lex}}}B}) \left\{ 
\begin{array}{lrr}
\min\displaystyle  \displaystyle\sum_{S \in [\mathcal{F}]^{\mathtt{deg}(R)}} b_S|S| & \notag \\
\displaystyle\sum_{S \in [\mathcal{F}]^{\mathtt{deg}(R)}} b_S \leq \mathtt{card}_\mathtt{lex}, & & (2)\\
    \displaystyle\sum_{S \in
    [\mathcal{F}]^{\mathtt{deg}(R)}} (I_B(S) - I_A(S)) v_S \ge 0, & & (3)\\
    \displaystyle\sum_{S \in [\mathcal{F}]^{\mathtt{deg}(R)}} (I_X(S) - I_Y(S)) v_S \ge 1 &\forall (X,Y) \in R, & (4)\\ 
    -b_SM \le v_S \le b_SM & \forall S \in [\mathcal{F}]^{\mathtt{deg}(R)}, &\\
    b_S \in \{0,1\} & \forall S \in [\mathcal{F}]^{\mathtt{deg}(R)}. &
\end{array}
\right.
\end{equation*}

A feasible solution of $\mathcal{P}_{A\succ^R_{\sqsubseteq_{\mathtt{lex}}}B}$ yields a model $\theta$ satisfying $\mathtt{deg}(\theta)\!=\!\mathtt{deg}_{\mathtt{lex}}$ (variables $b_S$ are only defined for $S\!\in\![\mathcal{F}]^{\mathtt{deg}(R)}$) and $\mathtt{card}(\theta)\!=\!\mathtt{card}_{\mathtt{lex}}$ (by constraint 2 on the value of $\mathtt{card}(\theta)$). Furthermore, constraint 3 ensures that $\theta\!\in\!\Theta^R_{B\succsim A}$, while constraint 4 ensures that $\theta\!\in\!\Theta^R$. 
If the optimal value of $\mathcal{P}_{A\succ^R_{\sqsubseteq_{\mathtt{lex}}}B}$ is $\mathtt{ws}_{\mathtt{lex}}$, then the corresponding model $\theta$ belongs to $\theta\!\in\!\Theta^{R}_{\sqsubseteq_{\mathtt{lex}}}$ (because then $(\mathtt{deg}(\theta), \mathtt{card}(\theta),\mathtt{ws}(\theta))\!=\!(\mathtt{deg}_\mathtt{lex}, \mathtt{card}_\mathtt{lex},\mathtt{ws}_\mathtt{lex})$), and thus there exists $\theta\!\in\!\Theta^{R}_{\sqsubseteq_{\mathtt{lex}}} \cap \Theta^R_{B\succsim A}$. Consequently:
\begin{itemize}[noitemsep,topsep=2pt]
    \item if the optimal value is strictly greater than $\mathtt{ws}_{\mathtt{lex}}$, or the polyhedron is empty, then $\Theta^{R}_{\sqsubseteq_{\mathtt{lex}}} \cap \Theta^R_{B\succsim A}\!=\!\emptyset$ and hence $A\!\succ^R_{\sqsubseteq_{\mathtt{lex}}}\!B$ (by Proposition~\ref{propROD});
    \item if the optimal value of $\mathcal{P}_{A\succ^R_{\sqsubseteq_{\mathtt{lex}}}B}$ is $\mathtt{ws}_{\mathtt{lex}}$, then $A\!\not\succ^R_{\sqsubseteq_{\mathtt{lex}}}\!B$.   
\end{itemize}

\section{Numerical tests}\label{sec:tests}

We call hereafter ORD the learning approach consisting in computing $(\mathtt{deg}(R),\mathtt{card}(R),\mathtt{ws}(R))$ and using $\succ^R_{\sqsubseteq_\mathtt{lex}}$ for preference prediction.
Numerical tests were carried out on Google Colab\footnote{two virtual CPU at 2.2GHz, 13GB RAM.}, with the aim of comparing ORD with state of the art approaches in two different settings: %
\begin{itemize}[noitemsep,topsep=2pt]
\item A first set of experiments were carried out on \emph{synthetic data}, i.e., obtained by simulating a user. They aimed at evaluating our approach in an ideal setting where %
a $\theta$-additive model perfectly fits the preferences. 
\item A second set of experiments were carried out on real-world data for content-based filtering methods (more precisely, movies described by binary attributes). These tests aimed at evaluating how our approach deals with partially described alternatives (i.e., with possible ``collisions'' if two distinct alternatives share the same description), compared to other state of the art approaches.
\end{itemize}
In both sets of experiments, we start with a learning set of preferences. Based on this learning set, pairwise preference predictions are then requested on random pairs of alternatives (pairs not in the learning set). As said earlier, the model may not make a prediction if it is not robust enough given the available preference data (i.e., if there is no robust ordinal dominance). 

\subsection{The synthetic and real-world datasets}

The dataset consists of ratings assigned by a user (DM) on a set $\mathcal{A}$ of $N$ alternatives. Given a set $\mathcal{F}\!=\!\{a_1,\ldots,a_n\}$ of binary features, a learning set $\mathcal{A}_{train}$ consists of $k\!\le\!N$ ratings of alternatives in $\mathcal{A}$, where each alternative $A_i$ ($i\!=\!1,\ldots,N$) is described by a binary vector $\overrightarrow{A}_i\!=\!(A_i^1, \ldots A_i^n)$, with $A_i^j\!=\!1$ if $a_j\!\in\!A_i$, and $A_i^j\!=\!0$ otherwise. The user rating of $A_i$ is denoted by $r_i$. \begin{versiona}
Denoting by $D$ the binary matrix whose rows are the line-vectors $\overrightarrow{A}_i$, the learning set formally looks as follows:
\begin{equation*}
    D = \begin{bmatrix} 
    A_{1}^1 & \dots  & A_{1}^n\\
    \vdots & \ddots & \vdots\\
    A_{k}^1 & \dots  & A_k^n 
    \end{bmatrix} \qquad
    r = \begin{bmatrix} 
    r_1 \\
    \vdots \\
    r_k 
    \end{bmatrix}
    \end{equation*}
\end{versiona}
The set of known strict preferences is $R\!=\!\{(A_i,A_j)\!\in\!\mathcal{A}_{train}^2:r_i>r_j\}$.

\begin{versiona}
\begin{example} \label{ex:dataset}
An example of a dataset of $k\!=\!4$ alternatives $A_1\!=\!\{a_1,a_3\}$, $A_2\!=\!\{a_2,a_3\}$, $A_3\!=\!\{a_3\}$, $A_4\!=\!\{a_1,a_2,a_3\}$, described on the feature set $\{a_1,a_2,a_3\}$:
    \begin{equation*}
    D = \begin{bmatrix} 
    1 & 0 & 1 \\
    0 & 1 & 1  \\
    0 & 0 & 1  \\
    1 & 1 & 1 
    \end{bmatrix} \qquad
    r = \begin{bmatrix} 
    3 \\
    3 \\
    2 \\
    1
    \end{bmatrix}
    \end{equation*}
The corresponding set $R$ is $\{(A_1,A_3),(A_1,A_4),$ $(A_2, A_3), (A_2,A_4), (A_3, A_4)\}$.    
\end{example} 
\end{versiona}

The \emph{real-world data} consist of ratings of movies by users picked up from the IMDb dataset\footnote{\url{www.kaggle.com/datasets/gauravduttakiit/imdb-recommendation-engine}.}. This is a dataset of movie reviews that contains over 50k reviews. Each movie $A_i$ is described by a set of binary features $A_i^j$, and the ratings $r_i$ are integer values ranging from 1 to 10. The experiments were conducted with a dataset of 50 users (randomly sampled) who each rated at least $k\!=\!100$ movies. Each movie is described using a subset of $n\!=\!8$ binary features (corresponding to the main genres of the movie, e.g., ``adventure'', ``animation'', ``children'', ``comedy'', ``fantasy'', etc.). %

The \emph{synthetic data} are generated in two steps: first a $\theta$-additive function $f_{\theta, v}$ is randomly sampled, then a rating function is inferred from $f_{\theta, v}$. The procedure is precisely detailed in the following two paragraphs.

\paragraph{Sampling  a $\theta$-additive function $f_{\theta, v}$} %
For sampling a function $f_{\theta, v}$, we first sample a set $\theta$ and then we sample parameters $v_S$ for $S\!\in\!\theta$. %
More precisely, the generation of $\theta$ is achieved as follows. First, $\theta$ is initialised as the set of singletons $\{a_1\},\{a_2\},\ldots,\{a_n\}$, then we add $\lfloor \alpha \times (2^{n} - n) \rfloor$ subsets of attributes, where the coefficient $\alpha\!\in\![0,1]$ makes it possible to control the model's complexity: for $\alpha\!=\!0$, only the singletons are in $\theta$, which yields the simple additive utility model, and for $\alpha\!=\!1$, all subsets of attributes are present, wich yields the most general utility model. %
Each subset $S$ is sampled according to a parameter $p\!\in\!(0,1]$:
\begin{enumerate}[noitemsep,topsep=2pt]
    \item{Initialize $S$ as a singleton by uniformly sampling in $\mathcal{F}$.}
    \item{Uniformly sample another attribute in $\mathcal{F}$ and add it to $S$.}
    \item{Exit this process if $
    S\!=\!\mathcal{F}$.}
    \item{Exit this process with a probability $p$ otherwise go to 2.}
\end{enumerate}
The expected size of sets $S$ we sample is $\mathbb{E}[|S|] = 2 + (1-p-(1-p)^{n-1})/p$.
\begin{versiona}
Table~\ref{tab:sizeS} gives some hints on the expected size of each $S$ according to $p$.
\end{versiona} 
Once $\theta$ is set, we sample the parameters $v_S$ for each $S\!\in\!\theta$ with a normal distribution $\mathcal{N}(0,\sigma)$. The sampling  of $f_{\theta,v}$ thus depends on three parameters $p$, $\alpha$ and $\sigma$. In the tests, $p$ varies in $[0.1, 0.9]$, $\alpha$ in $[0.1, 0.5]$, and $\sigma\!=\!100$.%

\begin{versiona}
\begin{table}[hbtp]   
\footnotesize
    \begin{tabular}{cccccc}
    \hline
        $p$ & 0.2 & 0.4 & 0.6 & 0.8 & 1 \\
    \hline
        $\mathbb{E}[|S|]$ & 3.95 & 3.18 & 2.62 & 2.25 & 2.00 \\
    \hline
    \end{tabular}
    \centering
\caption{\label{tab:sizeS}Expected size of subsets $S$ w.r.t. $p$, with $n\!=\!4$.}
\end{table}
\end{versiona}

\begin{versiona}
\begin{example}
\label{ex:tests}
If $n\!=\!5$, $p\!=\!0.3$, $\alpha\!= \!0.1$, then $\lfloor 0.1(2^5\!-\!5) \rfloor\!=\!2$ subsets $S$ are sampled in addition to the singletons. This may yield the parameter values given in Table~\ref{tab:ex_parameter_values}.
\begin{table}[hbtp]%
\footnotesize
\begin{tabular}{|c|l|l|l|l|l|l|l|}
\hline
$S$ & $\{a_1\}$ & $\{a_2\}$ & $\{a_3\}$ & $\{a_4\}$ & $\{a_5\}$ & $\{a_1,a_3\}$ & $\{a_1,a_4,a_5\}$  \\ \hline
$v(S)$ & 148.85 & 186.75 & 90.60 & -86.12 & 191.00 &  80.24 & -26.80 \\\hline
\end{tabular}%
\centering
\caption{\label{tab:ex_parameter_values}Example of parameter values sampled with $n=4$, $p=0.3$, and $\alpha= 0.1$.}
\end{table}   
\end{example}
\end{versiona}

\paragraph{From $f_{\theta,v}$ to a rating function} %
A function $\TierFunction \!:\! \AlternativeSet\!\rightarrow\!\{1,\ldots,t\}$ simulates the ratings of the user (of which only a subset of examples $\TierFunction(A_i)\!=\!r_i$, for $i\!\in\!\{1,\ldots,k\}$, is known to the model). The definition of $\TierFunction$ from $f_{\theta,v}$ depends on a parameter $t$ defining the domain $\{1,\ldots,t\}$ of possible ratings. The range of scores $f_{\theta,v}(A)\!=\!\sum_{S \in \theta} v_S I_A(S)$ of alternatives $A$ is partitioned into $t$ equally-sized intervals $(v_{k-1},v_k]$ between the min score $v_0\!=\!\min_{A \in \AlternativeSet} f_{\theta,v}(A)$ and the max score $v_{t}\!=\!\max_{A \in \AlternativeSet} f_{\theta,v}(A)$.
The function $\TierFunction$ is then: 
\begin{equation*}
    \TierFunction(A) = \min\{1\le k\le t : f_{\theta,v}(A) \leq v_k \}.
\end{equation*}
Put another way, the rating of $A$ corresponds to the index $k$ of the interval $(v_{k-1},v_k]$ in which $f_{\theta,v}(A)$ lies. %
In general, the wider the domain of possible ratings, the fewer incomparabilities (alternatives with the same rating).%

\begin{versiona}
\begin{example}
Coming back to Example~\ref{ex:tests}, let $\AlternativeSet\!=\!\{0,1\}^n$. We have that $\max_{A \in \mathcal{A}} f_{\theta, v}(A)\!=\!616.41$ and  $\min_{A \in \mathcal{A}} f_{\theta, v}(A)\!=\!-86.12$. Assume that one allows $t=3$ possible rating values. The intervals are then
$[-86.12, 148.05],$ $(148.05,382.23]$ and $(382.23,616.41]$. The alternative $A\!=\!\{a_2,a_3,a_4\}$ receives rating $2$ as $f_{\theta,v}(A)\!=\!191.23$ belongs to $(148.05,382.23]$.
\end{example}
\end{versiona}

\subsection{Baseline models}
\label{sec:baseline_models}

We briefly describe here the baseline models to which ORD is compared. Throughout the subsection, we have $\theta\!=\![\mathcal{F}]^{\mathtt{deg}(R)}$ and each alternative $A$ is described by an \emph{augmented} binary vector $\overrightarrow{A}\!=\!(I_{A}(S_1), \ldots, I_{A}(S_{|\theta|}))$, where $S_1,\ldots,S_{|\theta|}$ are the subsets of $\mathcal{F}$ of size less than or equal to $\mathtt{deg}(R)$.

\paragraph{Linear Regression (LR)} We consider the $\theta$-additive model, and we use linear regression to determine the value function $\hat{v}$ such that $f_{\theta,\hat{v}}$ best approximates the utility function $f$, by minimizing $\sum_{i=1}^k (\overrightarrow{A}_i\overrightarrow{v}^T\!-\! \textrm{\small normalized}(r_i))^2$, where $\textrm{\small normalized}(r_i)\!=\!\frac{r_i-\min_i r_i}{\max_i r_i-\min_i r_i}$ (note that $\overrightarrow{A}_i\overrightarrow{v}^T\!=\!f_{\theta,v}(A_i)$). Put another way, we use the least squares method\footnote{Precisely the \texttt{LinearRegression} function from the \texttt{scikit-learn} python library.} with ratings normalized in $[0,1]$. %
\begin{versiona}
\begin{equation*} %
    D = \begin{bmatrix} 
    I_{A_{1}}(S_1) & \dots  & I_{A_{1}}(S_\tau)\\
    \vdots & \ddots & \vdots\\
    I_{A_{k}}(S_1) & \dots  & I_{A_{k}}(S_\tau)
    \end{bmatrix} \qquad
    r = \begin{bmatrix} 
    \textrm{\small normalized}(r_1)\\
    \vdots \\
    \textrm{\small normalized}(r_k)
    \end{bmatrix}
\end{equation*}
\end{versiona}
We predict $A\!\succ\!B$ if $f_{\theta,\hat{v}}(\overrightarrow{A})\!>\! f_{\theta,\hat{v}}(\overrightarrow{B})$.

\begin{versiona}
\begin{example} Coming back to Example~\ref{ex:dataset}, the linear regression is performed on the following input data:
     \begin{equation*}
    D = \begin{bmatrix} 
    1 & 0 & 1 & 0 & 1 & 0\\
    0 & 1 & 1 & 0 & 0 & 1\\
    0 & 0 & 1 & 0 & 0 & 0\\
    1 & 1 & 1 & 1 & 1 & 1  
    \end{bmatrix} \qquad
    r = \begin{bmatrix} 
    1 \\
    1 \\
    0.5 \\
    0
    \end{bmatrix}
    \end{equation*}
    where the columns correspond to $S\!=\!\{a_1\},\{a_2\},\{a_3\},\{a_1,a_2\},\{a_1,a_3\},\{a_2,a_3\}$ in this order.
\end{example}
\end{versiona}

\paragraph{Support Vector Machine (SVM)} This baseline model is inspired by the approach proposed by \citeauthor{domshlak2005unstructuring} \cite{domshlak2005unstructuring}.
An SVM approach is a supervised learning method for binary classification: each example in the dataset is labeled by 0 or 1; an SVM is learned from the dataset\footnote{We use the \texttt{SVC} function from the \texttt{scikit-learn} python library.}, from which labels are inferred for new examples. In our setting, each preference $A\!\succ\!B$ in $R$ yields two examples: a $(|\theta|\!+\!1)$-dimensional vector $(\overrightarrow{A}\!-\!\overrightarrow{B}, 1)$ and another vector $(\overrightarrow{B}-\overrightarrow{A}, 0)$.
That is, the third component of $(\overrightarrow{A}\!-\!\overrightarrow{B}, c)$ is $c\!=\!1$ if $A$ is preferred to $B$, and $c\!=\!0$ if it is not. For predicting the preference between two alternatives $A$ and $B$, we infer the labels of $(\overrightarrow{A}\!-\!\overrightarrow{B})$ and $(\overrightarrow{B}\!-\!\overrightarrow{A})$ by using the SVM. If the label of $(\overrightarrow{A}\!-\!\overrightarrow{B})$ is 1 (resp. 0) and that of $(\overrightarrow{B}\!-\!\overrightarrow{A})$ is 0 (resp. 1), then we predict $A\!\succ\!B$ (resp. $B\!\succ\!A$). %

\begin{versiona}
\begin{example} Coming back to Example~\ref{ex:dataset}, the SVM approach is performed on the following input data:
     \begin{equation*} \label{eq:input_data_SVM}
    D = \begin{bmatrix} 
    1 & 0 & 0 & 0 & 1 & 0\\
    -1 & 0 & 0 & 0 & -1 & 0\\
    0 & -1 & 0 & -1 & 0 & -1\\
    0 & 1 & 0 & 1 & 0 & 1\\
    \vdots & \vdots & \vdots &  \vdots & \vdots & \vdots
    \end{bmatrix} \qquad
    y = \begin{bmatrix} 
    1 \\
    0 \\
    1 \\
    0 \\
    \vdots
    \end{bmatrix}
    \end{equation*}
    where the rows correspond to $\overrightarrow{A}_1\!-\!\overrightarrow{A}_3,\,\overrightarrow{A}_3\!-\!\overrightarrow{A}_1,\,\overrightarrow{A}_1\!-\!\overrightarrow{A}_4,\,\overrightarrow{A}_4\!-\!\overrightarrow{A}_1,\ldots$ in this order.
\end{example}
\end{versiona}

\paragraph{K-Nearest Neighbours (KNN)}
The distance-based models are widely used in the context of recommender systems. The distance-based model we consider is implemented as follows. 
The predicted rating of an alternative $A$ is obtained by making a weighted sum $\sum_{i=1}^K w_i r_i$ of the ratings $r_1,\ldots,r_K$ of its $K$ nearest neighbours $A_1,\ldots,A_K$ in the learning set\footnote{We use the \texttt{ KNeighborsClassifier} function from the \texttt{scikit-learn} python library.}, with each weight $w_i$ proportional to the Euclidean distance of the neighbour $\overrightarrow{A}_i$ to $\overrightarrow{A}$.
The value of $K$ was set to $K\!=\!5$ in our experiments, after preliminary tests showing this was the value yielding the best results for the dataset considered here. For predicting the preference between two alternatives $A$ and $B$, we compute the predicted ratings of them, and predict the preference accordingly.

\subsection{Experimental setup}
\label{sec:exp_setup}

In all experiments, the dataset
is a set $\mathcal{A}$ of $N$ alternatives, described by a set $\mathcal{F}$ of $n$ binary features, and an associated rating vector $r$ (integer values). %
The rating $r(A)$ of each alternative $A\!\in\!\mathcal{A}$ is known.
To compare the performances of the different learning methods, we extract a subset $\mathcal{A}_{train}$ of $k$ alternatives from $\mathcal{A}$, on which the models are trained. 
The alternatives in $\mathcal{A}_{train}$ are chosen uniformly at random.
We then randomly sample 100 pairs $\{A,B\}$ in $\mathcal{A}$ such that $A\!\not\in\!\mathcal{A}_{train}$ or $B\!\not\in\!\mathcal{A}_{train}$ (possibly neither $A$ nor $B$ belongs to $\mathcal{A}_{train}$), and we compare the predicted pairwise preference with the actual preference: $A\!\succ\!B$ if $r(A)\!>\!r(B)$, $B\!\succ\!A$ if $r(B)\!>\!r(A)$, $A\!\sim\!B$ (incomparability) if $r(A)\!=\!r(B)$. The extraction of a subset $\mathcal{A}_{train}$ from $\mathcal{A}$, the training of each model and the (100) pairwise preference predictions are performed 10 times, and the prediction performances are averaged over the 10 runs. We detail below the parameters that are used for the experiments on synthetic data and for the experiments on real-world data. %

\paragraph{Synthetic data} The experiments on synthetic data were conducted with $|\mathcal{F}|\!=\!8$ binary features, which yields a set $\mathcal{A}$ of $2^{|\mathcal{F}|}\!=\!256$ alternatives, a scale of $t\!=\!12$ possible ratings, and the set of parameters $(\alpha,p,\sigma)\!=\!(0.1,0.9,100)$ for the generation of $f_{\theta,v}$. This set of parameters yields functions $f_{\theta,v}$ that are usually up to 4-additive, with an average $|\theta|$ equal to 12. This setting is not really restrictive as, given the number of strict pairwise preferences in $R$ that are considered in our experiments (i.e., $|R|\!\leq\!{|\mathcal{A}_{train}| \choose 2}$), it is unlikely that $R$ cannot be represented by using a function $f_{\theta,v}$ of degree up to 4. 
The size of $\mathcal{A}_{train}$ indeed varies between 12 and 29, from which between $|R|\!=\!{12 \choose 2}\!=\!66$ and ${29 \choose 2}\!=\!400$ pairwise preferences can be inferred.

\paragraph{Real-world data} 
For each of the 50 users that have rated at least 100 movies, a dataset $\mathcal{A}$ including between 45 and 100 alternatives is first extracted. A training set $\mathcal{A}_{train}$ is then extracted from $\mathcal{A}$, with $|\mathcal{A}_{train}|$ corresponding to 90\% of $|\mathcal{A}|$ (which is common practice in machine learning, in particular for performing 10-fold cross-validation). The size of $\mathcal{A}_{train}$ thus varies from 5 to 10, from which between $|R|\!=\!{5 \choose 2}\!=\!10$ and ${10 \choose 2}\!=\!45$ pairwise preferences can be inferred.

\subsection{Evaluation metrics}
We outline here the specific metrics that will be used to evaluate the ORD approach and compare it to other methods.
To define our metrics we consider the 9 cases that can occur in the confusion matrix defined below. 
\paragraph{Confusion Matrix}
For a given pair of alternatives $(A,B)\!\in\! \AlternativeSet^2$ each model could either infer (predicted output) that $A$ is \emph{better} than $B$ ($A \succ B$), or that $A$ is \emph{worse} than $B$ ($B \succ A$) or it could return that the relation between $A$ and $B$ is \emph{unknown}. 
Then, as outlined earlier, by comparing $r(A)$ and $r(B)$, we can have (real outputs) that $A$ is indeed better than $B$ if $r(A) > r(B)$ or that $A$ is worse than $B$ if $r(B) > r(A)$ or that the relation between them is unknown if they share the same rating (incomparability). 
Our metrics are based on the confusion matrix defined in Table~\ref{TableCases}, where the rows symbolizes the predicted outputs and the columns the real outputs. 

\begin{table}[!h]
\centering
\footnotesize
\begin{tabular}{|c|c|c|c|}
\hline
Predicted/Real & (B)etter & (W)orst & (U)nknown \\ \hline
(B)etter          & BB       & BW      & BU        \\ \hline
(W)orst           & WB       & WW      & WU        \\ \hline
(U)nknown         & UB       & UW      & UU        \\ \hline
\end{tabular}
\caption{Confusion matrix.\label{TableCases}}
\end{table}

\paragraph{Precision} The precision is defined as the ratio between the number of correct predictions among all the predictions that \emph{were} made. 
$$
P = \frac{BB + WW}{BB + WW + BW + WB + BU + WU}.
$$

\paragraph{Recall} The recall is defined as the ratio between the number of correct predictions among all the predictions that \emph{could be} made. 
$$
R = \frac{BB + WW}{BB + WW + BW + WB + UB + UW}.
$$

The precision metric penalizes the models making unreliable predictions, while the recall metric penalizes the models that avoid making predictions.

\paragraph{F1-score}
F1-score is a metric that combines precision and recall to provide a balanced evaluation of a model's performance. It is obtained by computing the harmonic mean of precision and recall:
$$
F = 2 \frac{P \times R}{ P + R}.
$$
As the F1-score captures both precision and recall, it is an ideal metric for evaluating the robustness and accuracy of the studied models. Hence, we strongly rely on it when presenting our results.

\paragraph{Prediction Correctness} This metric is similar to precision, except that it does not take into account predictions that cannot be evaluated for lack of preferential information to check whether they are correct or incorrect.
$$
PC = \frac{BB + WW}{BB + WW + BW + WB}.
$$

\paragraph{Prediction Rate} This metric does not take into account the correctness of the predictions, it simply evaluates the rate at which the model produces predictions:
$$
PR = 1- \frac{UB+UW+UU}{M},
$$
where $M$ represents all the cases of Table \ref{TableCases} ($BB + WW + BW + WB + BU + WU + UB + UW + UU$).

\subsubsection{Results on synthethic data}

\begin{figure}[!h]
    \centering
    \includegraphics[scale = 0.45]{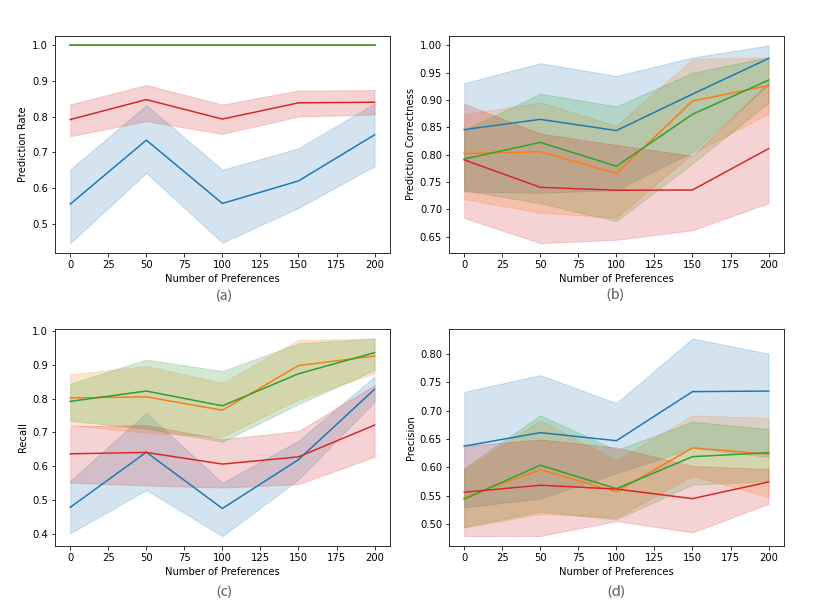}
    \caption{Precision, Recall, Prediction Rate and Prediction Correctness according to the number $|R|$ of preferences for models ORD (blue), KNN (red), SVM (orange), LR (green).}
    \label{fig:metrics_synthethic}
\end{figure}

\begin{figure}[!h]
    \centering
    \includegraphics[scale = 0.45]{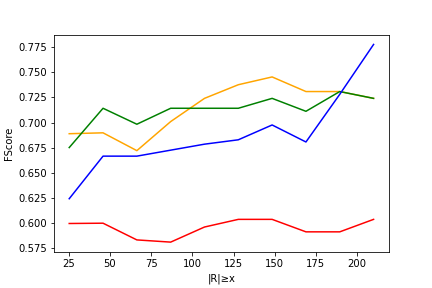}
    \caption{Average F1-score according to the threshold x on the number of preferences in $R$, for models ORD (blue), KNN (red), SVM (orange), LR (green).}
    \label{fig:fscore}
\end{figure}

The results on synthetic data are presented in Figures~\ref{fig:metrics_synthethic} and~\ref{fig:fscore}, where the x-axis gives the number of preferences in $R$ (inferred from the ratings of the alternatives in $\mathcal{A}_{train}$) and the curves show the mean and 95\% confidence interval. The curves show how the different metrics evolve with $|R|$.

Figure~\ref{fig:metrics_synthethic} shows that each approach produces a different compromise between the number of predictions and their quality. 
The LR and SVM approaches have, by design, %
a prediction rate of 1 (the orange line is covered by the green one in the figure) but with predictions that are always less accurate than the predictions made by ORD. 
Since the KNN approach averages the rates of the $K$ nearest neighbours of the instance to predict, it may occur that two alternatives obtain the exact same score (e.g., if the $K$ nearest neighbours are the same for both alternatives) and thus that no strict preference prediction is made. The prediction rate of KNN is 0.8 on average, but the curve of prediction correctness (and thus the curves of recall and precision) shows that it does not improve the accuracy of the predictions compared to the other methods, quite the contrary. 

The ORD model, in contrast, outperforms the other models in terms of precision, as illustrated by the average prediction correctness that is almost always above 0.85.
However, since the recall metric penalizes the models that do not make enough predictions, the performances of ORD are below the average performance of the other models in terms of recall. 
This behavior is, in a sense, intrinsic to an approach that prioritizes the robustness of predictions.
Nevertheless, as can be seen in Figure~\ref{fig:metrics_synthethic}c, the recall significantly improves with $|R|$.
There are a few irregularities in the curve of the prediction rate for ORD, due to the fact that $\mathtt{deg}(R)$ grows in steps with $|R|$, and this degree impacts $|\Theta^R|$ and thus the number of predictions made (the ordinal dominance relationship becoming more stringent).

While the interest of a compromise between quantity and quality of predictions inherently varies depending on the specific context of an application, the F1-score is a commonly adopted metric to navigate these trade-offs. Figure~\ref{fig:fscore} shows the average F1-score on learning instances where $|R|\!\ge\!\mathrm{x}$, in function of $\mathrm{x}$. 
We observe that the average F1-scores of ORD, LR and SVM are close for $\mathrm{x}\!=\!50$ (i.e., $R$ include at least 50 preferences). Notably, as $\mathrm{x}$ grows so that $R$ encompasses at least 170 preferences, the ORD approach demonstrates a significant performance advantage over LR and SVM.

The curves in Figure~\ref{fig:running_times_synthethic} gives the average running times of ORD (in seconds, averaged over 20 instances) according to the number $n$ of features (for $300\!\leq\!|R|\!\leq\!400$) and the number $|R|$ of known strict pairwise preferences (for $n\!=\!8$). The orange curve gives the average running time for \emph{one} pairwise preference prediction; this is the most time-consuming phase: note indeed that learning $(\mathtt{deg}(R),\mathtt{card}(R),\mathtt{ws}(R))$ is only performed once for each $R$, while 100 preference predictions are made for each $R$ in our tests.

\begin{figure}[!h]
    \centering
    \includegraphics[scale = 0.45]{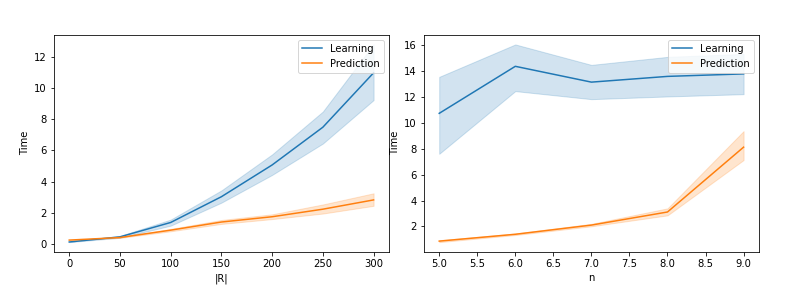}
    \caption{Running times of ORD (in seconds).}
    \label{fig:running_times_synthethic}
\end{figure}

\subsubsection{Results on real-world data}

The results obtained on real-world data from IMDb are summarized in Table~\ref{tab:real}. 
Compared to the synthetic data, the precision rates of KNN, LR and SVM significantly decrease, while the precision rate of ORD is holding up better. The recall of ORD remains lower than the recall of LR and that of SVM, but this is overcompensated by the reduced precision performance gap between ORD and LR/SVM. This allows ORD to achieve a better compromise between precision and recall, thus yielding a better F1-Score.

\begin{table}[h!]
  \centering
  \footnotesize
  \caption{Model performances averaged on all the users}
  \begin{tabular}{lrrrr}
    \toprule
    Model & Prediction Rate & Precision & Recall & F1-Score \\
    \midrule
    KNN & 0.82 & 0.48 & 0.65 & 0.59  \\
    LR & 1  & 0.55 & 0.90 & 0.69  \\
    ORD & 0.60 & 0.76 & 0.83 & 0.81  \\
    SVM & 1 & 0.55 & 0.92 & 0.70  \\
    \bottomrule
  \end{tabular}
  \label{tab:real}
\end{table}

\section{Conclusion}

We have presented here a robust ordinal method for subsets comparisons with interactions. The model we use is not restrictive, in the sense that any strict weak order on subsets can be represented. The learning method achieves a trade-off between the number of predicted preferences and the accuracy of the predictions, by relying on a robust ordinal dominance relation between subsets.

Several research directions are worth investigating, among which the adaptation of the approach to an active learning setting where one interactively determines a sequence of queries to minimize the cognitive burden for the decision maker, or a better consideration of potential ``errors'' in the preferences used as a learning set.

\section*{Acknowledgements}
We acknowledge the support of the French Agence Nationale de la Recherche (ANR), under grant ANR20-CE23-0018 (project THEMIS).

 \bibliographystyle{apalike}
 \bibliography{cas-refs}

 \newpage

 \appendix

 \section{Properties of the $\theta$-ordinal dominance relation}

\begin{restatable}{proposition}{PropPropertiesOrdDom}
The following properties hold for $\succ^R_{\theta}$:
    
    $~~~$ (i) $\succ_{\theta}^R$ is asymmetric. 
    
    $~~~$ (ii) $\succ_{\theta}^R$ may not be complete.

    $~~~$ (iii) $\succ_{\theta}^R$ is not necessarily negatively-transitive. 
\end{restatable}

\begin{proof}
    $(i)$ $A \succ_\theta^R B \Rightarrow \forall v \in V_\theta^R,\,f_{\theta, v}(A) > f_{\theta, v}(A)$. Thus there is no function $v'\!\in\!V^R_\theta$ such that $f_{\theta, v'}(A) < f_{\theta, v'}(A)$. %
    
    $(ii)$ As shown in Example~\ref{ex : intro theta model 2}, we may have $v, v' \in V_\theta^R$ such that $f_{\theta, v}(A) > f_{\theta, v}(B)$ and $f_{\theta, v'}(B) > f_{\theta, v'}(A)$. We have then neither $A \succ_\theta^R B$ nor $B \succ_\theta^R A$, and thus $\succ_\theta^R$ may not be complete.

    $(iii)$ Let $\mathcal{F} = \{a_1,a_2,a_3\}$, $R = \{ (\{a_1\},\{a_3\})\}$, $\theta = \{\{a_1\},\{a_2\},\{a_3\}\}$ and $v, v'$ two value functions defined as follows: 
    \begin{align*}
        v(\{a_1\}) = 2 &, v(\{a_2\}) = 3, v(\{a_3\}) = 1, \\
        v'(\{a_1\}) = 3 &, v'(\{a_2\}) = 1, v'(\{a_3\}) = 2.
    \end{align*}
    We have that $v, v'\!\in\!V_\theta^R$ as $f_{\theta, v}(\{a_1\})\!>\!f_{\theta, v}(\{a_3\}) $ and $f_{\theta, v'}(\{a_1\})\!>\!f_{\theta, v'}(\{a_3\}) $.\\
    It follows from $f_{\theta, v}(\{a_2\}) > f_{\theta, v}(\{a_1\})$ that $\neg (\{a_1\} \succ_\theta^R \{a_2\})$. \\
    It follows from $f_{\theta, v'}(\{a_3\}) > f_{\theta, v'}(\{a_2\})$ that $\neg (\{a_2\} \succ_\theta^R \{a_3\})$. \\
    Yet $\{a_1\} \succ_\theta^R \{a_3\}$ by definition of $R$.
\end{proof}

\begin{restatable}{proposition}{propChangeTheta} \label{propChangeTheta}
Given a set $R$ of strict pairwise comparisons, and $\theta\!\in\!\Theta^R$, if $R'\!\subseteq\!R$ then: (i) $\theta \!\in\!\Theta^{R'}$; (ii) $A \!\succ_{\theta}^{R'}
\!B\!\Rightarrow\!A\!\succ_{\theta}^{R}\!B$; (iii) $A\!\succ_{\theta}^{R}\!B\!\Rightarrow\! \neg (B \succ_{\theta}^{R'} A)$.
\end{restatable}

\begin{proof} 
$(i)$ If all the preferences in $R$ can be represented by a $\theta$-additive function, then so can the preferences in $R'$ as $R'$ is compounded of a subset of the preferences in $R$. 

$(ii)$ If the preferences in $R'$ imply that $A$ should be necessarily strictly preferred to $B$, then $R$ will imply the same conclusion as $\Theta^R\!\subseteq\!\Theta^{R'}$ (because $R$ contains all the preference constraints in $R'$, along with additional constraints). 

$(iii)$ The contrapositive is proved as follows: $B\!\succ_\theta^{R'}\!A\!\Rightarrow\!B\succ_\theta^R\!A$ by $(ii)$, and $B\succ_\theta^R\!A\!\Rightarrow\!\neg(A\succ_\theta^R\!B)$ because strict preferences are asymmetrical. 
\end{proof}

\begin{restatable}{proposition}{propChangeThetaBis}\label{propChangeThetaBis}
Let $\theta,\theta'\!\in\!\Theta^R$. If $\theta'\!\subseteq\!\theta$, then the following assertions hold:

(i) $A\!\succ_{\theta}^R\!B\!\Rightarrow\!A \!\succ_{\theta'}^R\!B$; (ii) $A \!\inc_{\theta'}^R\!B\!\Rightarrow\!A\! \inc_{\theta}^R\!B$; (iii) $A\! \succ_{\theta'}^R\!B\!\Rightarrow\!\neg (B \!\succ_{\theta}^R\!A)$.
\end{restatable}

\begin{proof}
$(i)$ is true because if $f_{\theta, v}(A) > f_{\theta, v}(B)$ for all $v \in V_{\theta}^R$, then we should also have $f_{\theta', v}(A) > f_{\theta', v}(B)$ for all $v \in V_{R}^{\theta'}$. Indeed, each element of $V_{\theta'}^{R}$ can be seen as a value function in $V_{\theta}^R$ in which the parameters $v_S$ are set to 0 for $S\in \theta\setminus \theta'$.

$(ii)$ follows by a similar argument as for $(i)$. 

$(iii)$ The contrapositive is proved as follows: $B\!\succ_\theta^{R}\!A\!\Rightarrow\!B\succ_{\theta'}^R\!A$ by $(i)$, and $B\succ_{\theta'}^R\!A\!\Rightarrow\!\neg(A\succ_{\theta'}^R\!B)$ because strict preferences are asymmetrical.
\end{proof}

\end{document}